%% file: main.tex
\documentclass{article}

\usepackage{microtype}
\usepackage{graphicx}
\usepackage{booktabs} %

\usepackage{hyperref}

\usepackage[table,xcdraw]{xcolor}

\usepackage[accepted]{icml2022}

\usepackage{enumitem}
\usepackage{caption}
\usepackage{wrapfig}
\usepackage{url}
\usepackage{letltxmacro}
\LetLtxMacro{\originaleqref}{\eqref}
\usepackage{tikz}
\usepackage{subcaption}
\usepackage{amsmath}
\usepackage{amssymb}
\usepackage{amsfonts}
\usepackage{amsthm}
\usepackage{todonotes}
\usepackage{lscape}
\usepackage{float}

\def\R{\mathbb{R}}

\def\N{\mathcal{N}}
\def\U{\mathcal{U}}

\def\I{I}

\newcommand{\DD}{X}  %
\newcommand{\C}{\mathcal{C}}
\newcommand{\abs}[1]{\left\lvert{#1}\right\rvert}
\newcommand{\set}[1]{\left\{{#1}\right\}}
\newcommand{\norm}[1]{\lVert{#1}\rVert}
\newcommand{\pip}{\pi^\perp}
\providecommand{\TO}{\textbf{to}\ }

\newtheorem{theorem}{Theorem}[section]
\newtheorem{proposition}{Proposition}[section]
\newtheorem{lemma}{Lemma}[section]

\theoremstyle{definition}
\newtheorem{definition}{Definition}[section]

\icmltitlerunning{LIDL: Local Intrinsic Dimension Estimation Using Approximate Likelihood}

\begin{document}

\twocolumn[
\icmltitle{LIDL: Local Intrinsic Dimension Estimation Using Approximate Likelihood}

\icmlsetsymbol{equal}{*}

\begin{icmlauthorlist}
\icmlauthor{Piotr Tempczyk}{uw,op,dt}
\icmlauthor{Rafał Michaluk}{uw,op}
\icmlauthor{Łukasz Garncarek}{op,ap}
\\
\icmlauthor{Przemysław Spurek}{uj}
\icmlauthor{Jacek Tabor}{uj}
\icmlauthor{Adam Goli{\'n}ski}{ox}
\end{icmlauthorlist}

\icmlaffiliation{uw}{Institute of Informatics, University of Warsaw}
\icmlaffiliation{op}{Polish National Institute for Machine Learning (\url{www.opium.sh})}
\icmlaffiliation{dt}{\url{deeptale.ai}}
\icmlaffiliation{ox}{University of Oxford}
\icmlaffiliation{uj}{GMUM, Jagiellonian University}
\icmlaffiliation{ap}{Applica}

\icmlcorrespondingauthor{Piotr Tempczyk}{piotr.tempczyk@mimuw.edu.pl}

\icmlkeywords{Intrinsic Dimension, Density Estimator, Local Intrinsic Dimension Estimation, Normalizing Flows}

\vskip 0.3in
]

\printAffiliationsAndNotice{}  %

\input{abstract}
\input{introduction}

\input{method}

\input{empirical}
\input{related}

\input{experiments}

\input{conclusions}

\newpage
\bibliography{bibliography}
\bibliographystyle{icml2022}
\input{supplement}

\end{document}

%% file: abstract.tex
\begin{abstract}
Most of the existing methods for
estimating the local intrinsic dimension of 
a data distribution 
do not scale well to high dimensional data. 
Many of them rely 
on a non-parametric nearest neighbours approach
which suffers from the curse of dimensionality.
We attempt to address that challenge
by proposing a novel approach to the problem: Local Intrinsic Dimension estimation using approximate Likelihood 
(LIDL).
Our method relies on an arbitrary density estimation method as its subroutine, 
and hence tries to sidestep the dimensionality challenge by making use of the recent progress in \emph{parametric} neural methods for likelihood estimation.
We carefully investigate the empirical properties of the proposed method, 
compare them with our theoretical predictions,
show that LIDL yields competitive results on the standard benchmarks for this problem,
and that it scales to thousands of dimensions.
What is more, we anticipate this approach to improve further with the continuing advances in the density estimation literature.

\end{abstract}

%% file: introduction.tex
\section{Introduction}
\label{sec:introduction}

\begin{figure}[t]
    \centering
    \par\vspace{-0pt}\par
    \includegraphics[trim=0 0 0pt 0,clip,width=\linewidth]{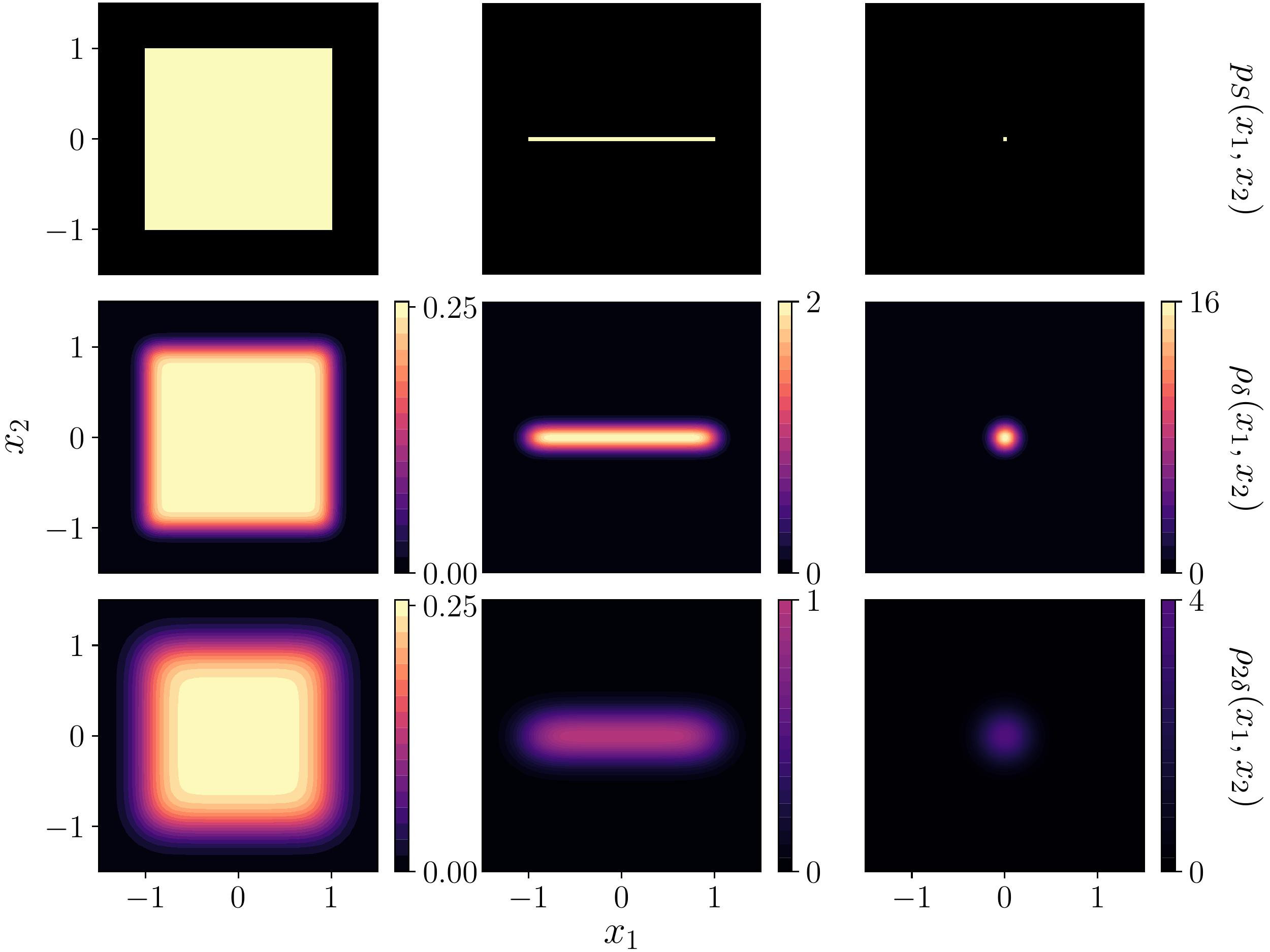}
    \caption{Illustration of LIDL's core insight. [Top] Three uniform distributions $p_S$ supported respectively on a square, interval, and a point, with intrinsic dimensions $2$, $1$, $0$. 
    [Middle/bottom] Perturbed densities $\rho_{\delta}$ and $\rho_{2\delta}$ resulting from addition of Gaussian noise with different noise magnitudes: $\delta$ and $2\delta$. 
    Our core insight is that the difference between the densities $\rho_{\delta}(x)$ and $\rho_{2\delta}(x)$ at any point $x$ depends on the local intrinsic dimension (LID) at that point. Consider point $x\!=\!(0,0)$. For the left column, that difference is zero; for the middle one, the density is halved; for the right one, it is quartered. We leverage this mechanism to estimate LID.}
    \label{fig:method-1}
\end{figure}

In this paper, we consider the  problem  of local intrinsic dimension (LID) estimation of a lower-dimensional data manifold embedded in a higher-dimensional ambient space.

Intrinsic dimension estimation is an established problem in data analysis and representation learning \citep{ansuini2019intrinsic,li2018measuring,rubenstein2018latent}.
It was studied in the context of dimensionality reduction, clustering, and classification problems \citep{vapnik2013nature,kleindessner2015dimensionality,camastra2016intrinsic} and
some prototype-based clustering algorithms \citep{claussen2005magnification,struski2018lossy}, among others. 
It is also a powerful analytical tool to study the process of training and representation learning in deep neural networks \citep{li2018measuring, ansuini2019intrinsic}.
\citet{rubenstein2018latent} shows how the mismatch between the latent space dimensionality and the dataset's ID may hurt the performance of auto-encoder generative models like VAE~\citep{kingma2014auto}, WAE \citep{tolstikhin2017wasserstein}, or CWAE \citep{knop2020cramer}.
Recent results of \citet{pope2020intrinsic} show that the global ID of the dataset impacts the training process of a machine learning model, sample efficiency, and its ability to generalize. 

Recently, there has also been a rise of interest in methods for simultaneous manifold learning and density estimation \citep{brehmer2020flowsa,caterini2021rectangular,ross2021tractable}.
Most of these methods require setting the manifold dimension as a hyperparameter and the standard practice is to do this heuristically or by performing a hyperparamater sweep.
Intrisic dimensionality estimation methods offer a principled alternative to this practice.

Outside of machine learning, an area where we hope reliable ID estimation might help is in the fields of science and engineering, where we nowadays collect enormous amounts of high-dimensional data.
The first challenge such modern datasets present for LID estimation is scalability: we seek algorithms which yield unbiased estimates in high-dimensional spaces.
Many of the existing methods for estimating ID rely on a non-parametric nearest neighbours-based approach which suffers from the curse of dimensionality.
They estimate the LID by investigating the distribution of local distances and they run into the problem of ''boundary effects, which distort the distribution and lead to a negative bias especially for high dimensions where any manifold with a boundary has almost all of its volume concentrated close to the boundary'' \citep{johnsson2014low}, which is a symptom of the curse of dimensionality.
ESS method \citep{johnsson2014low} itself sidesteps this challenge by using angular rather than distance information, leading to its competitive performance in high dimensions, as highlighted in our evaluations.
In contrast, the method we propose uses a \emph{parametric} density estimation model which is a different mechanism to sidestep this challenge.

The second challenge is being able to deal with datasets that have highly non-isotropic structure 
(we refer to them as \emph{multiscale}). 
A good example of this kind of dataset are images of a human face.
In those datasets we have many latent factors of variation and some of them account for much less ambient-space variance than others. 
For instance, let us consider two of them: eye and head rotation.
Each of them induces 2-dimensional manifold in the data space, but variance coming from the latter factor is much larger. 

The third challenge is being able to deal with dataset dimensionality on different scales, which includes dealing with the ambient-space noise in the data. Most of the data we use is affected by measurement noise, and all the datasets are deformed further by being quantized and stored as finite precision numbers. 
To understand how it impacts LID estimation let us imagine a simple physical problem: we have a particle moving in an electromagnetic field, and our dataset is the set of 2-dimensional vectors describing its positions on a plane. 
Additionally, our measurements are quantized with some very small quantization step. 
When we magnify the dataset to the scale of quantization step, we see the dataset as a 0-dimensional manifold--it lies on a finite grid. When we zoom out to the scale of the measurement noise, we observe a 2-dimensional point cloud: one bigger dimension along the trajectory of the particle, and the other--smaller one--coming from the imperfect measurement. When we zoom out far enough, we observe only the 1-dimensional trajectory of that particle. It would be very convenient to be able to easily set an operating scale of an algorithm, e.g. have a possibility to ignore the dimensions that are artificially created by the measurement noise.

To address these problems, we propose a new method for Local Intrinsic Dimension estimation using approximate Likelihood (LIDL). Instead of using non-parametric methods based on local samples from the neighbourhood of a given point, it is based on parametric probability density estimation, which scales better to higher-dimensional setting. 
At the core of our method lies the observation that when we add Gaussian noise $\N(0,\delta^2 I)$ to the dataset $X$ embedded in $\R^D$, the rate of change of the log-likelihood at $x \in X$ (at which LID equals $d$) is approximately linear in the logarithm of $\delta$. Moreover, the proportionality constant is $\beta \approx  d - D$ and we can estimate it using linear regression, thus estimating $d$. 
We may view $\delta$ as a scale parameter of our method, which may be of practical benefit beneficial when dealing with noisy datasets. Intuitive visualisation of this concept can be found in Fig.~\ref{fig:method-1}, and the formal derivation can be found in Sec.~\ref{sec:method}.

To the best of our knowledge, LIDL is the first theoretically grounded method of LID estimation that uses global density estimation methods. In our method we relax the assumptions many algorithms make about uniformity of the density and the manifold flatness in the neighbourhood of $x$. We show theoretically and experimentally that we can deal with manifolds consisting of multiple multiscale connected components of different dimensions. We also compare our algorithm with a wide range of other LID estimation algorithms, and verify that only our algorithm can give unbiased estimates for high-dimensional datasets. 
The code to reproduce our results is available at \url{github.com/opium-sh/lidl}.

The empirical success of our method was made possible by using neural density estimators called \emph{normalizing flows} (NF) \citep{rezende2015variational}, which can estimate densities even in high-dimensional spaces as images. 
Although in this work we use NF as LIDL's density estimator, our method can be used with any density estimation method.
Thus, we anticipate that its capabilities to grow further with continuing progress in the area of density estimation.

\textbf{Contributions}.
Our main contribution is the introduction of a novel, accurate and scalable method of LID estimation. The method is backed up with solid mathematical foundations, and verified both theoretically and experimentally. The impact of LID itself on neural network performance is experimentally assessed. Finally, we identified some problems with existing approaches to LID estimation.

%% file: method.tex
\section{Method}
\label{sec:method}
In this section, we introduce the problem setting formally and we lay out the LIDL method theoretically, including its derivation and pointing out certain predictions about its behavior that we later verify empirically in Section~\ref{sec:empirical-evaluation}.

It is often known that a particular dataset $X$ is a subset of some data manifold $M$ equipped with probability measure $\mu$. However, the manifold $M$ and the dataset $X$ need not be directly observable. Instead, there may exist an embedding (or, more generally, an immersion satisfying some regularity conditions, see Appendix~\ref{sec:non-connected}) $j\colon M\to\R^D$ into a Euclidean space of larger dimension, through which we can view $X$. 

The method we propose is based upon the observation, expressed in Theorem~\ref{thm:core-estimate}, that for a probability measure supported on an embedded submanifold of an Euclidean space, the dimension of its support can be recovered from its asymptotic behavior under small normally distributed perturbations (see Fig.~\ref{fig:method-1} for an intuitive illustration). 

\subsection{The formal setting}

We first define a class of measures we will restrict our considerations to, which we will call \emph{smooth} measures, including all measures with continuous positive densities.

\begin{definition}
\label{def:smooth-measure}
    A positive measure $\nu$ on a manifold $N$ will be called \emph{smooth} if for any chart $\psi\colon U\to V\subset\R^n$ of $N$, the pushforward $\psi_*\nu$ is absolutely continuous with respect to the Lebesgue measure $\lambda$ on $V$, and moreover, its density is locally bounded away from $0$, i.e.\ any $x\in V$ admits a neighborhood on which $dj_*\nu/d\lambda > c$ for some $c>0$.
\end{definition}

Let $S\subset\R^D$ be a smooth connected  $d$-dimensional embedded submanifold of a high-dimensional Euclidean space $\R^D$ (the more general case of a non-connected immersed manifold is dealt with in Appendix~\ref{sec:non-connected}).
This is our observable data manifold, embedded in Euclidean space, i.e.\ $S=j(M)$. Furthermore, suppose we are given a smooth (according to Definition~\ref{def:smooth-measure}) probability measure $p_S$ on $S$, representing the data probability distribution. In our notation, this is the pushforward of the probability $\mu$ on $M$, i.e.\ $p_S = j_*\mu$. We will implicitly treat $p_S$ as a probability distribution on the whole ambient space $\R^D$.

The Gaussian function (i.e.\ the density of the standard normal distribution) on a Euclidean space $V$ will be denoted by $\phi^V$, or $\phi^n$ in the case where $V$ is the standard $\R^n$ space. Also, for $\delta > 0$, let
\begin{equation}
\label{eq:dilated-gaussian}
    \phi^V_{\delta}(x) = \delta^{-\dim V}\phi^V(x/\delta)
\end{equation} 
be the density of the normal distribution $\N(0, \delta^2 I)$ with covariance matrix $\delta^2 I$, where $I$ is the identity matrix on $V$. 
 
Under the above notation, if $X\sim p_S$ is a random vector representing the data, and $N_\delta\sim\N(0,\delta^2 I)$ a normally distributed random noise vector, the distribution of the perturbed random vector $X+N_\delta$ in $\R^D$ is given by the convolution $p_S * \N(0,\delta^2 I)$, and has density
\begin{equation}
\label{eq:mollified-density}
    \rho_\delta(x) = \int_S \phi^D_{\delta}(x-y)\, d p_S(y).
\end{equation}

Finally, let us introduce a notation for uniform multiplicative estimates. We will write that $f(x,y) \asymp g(x,y)$ uniformly in $x$ if for every $y$ there exists $C>0$ such that for all $x$
\begin{equation}
    C^{-1}g(x,y) \leq f(x,y) \leq Cg(x,y).
\end{equation}
This notation extends to any number of variables. We will use it to declutter the proofs from irrelevant 
constants.

\subsection{The core estimate}
\label{sec:core-estimate}

At any $x\in S$ the tangent space of $\R^D$ admits a decomposition $T_x\R^D = T_xS \oplus N_xS$ into a direct sum of the tangent and normal spaces of $S$. Under the natural identification of $T_x\R^D$ with the underlying $\R^D$ (mapping the origin of $T_xS$ to $x$), the tangent and normal spaces of $S$ at $x$ become two affine subspaces of $\R^D$ intersecting at $x$. Denote by $\pi_x\colon \R^D\to T_xS$ and $\pip_x\colon \R^D\to N_xS$ be the corresponding orthogonal projections. With this notation, the following decomposition of the Gaussian density holds
\begin{equation}
\label{eq:gaussian-decomposition}
    \phi^D_\delta(x-y) = \phi^{T_xS}_\delta(\pi_x(y)) \phi^{N_xS}_\delta(\pip_x(y)).
\end{equation}

By the Inverse Function Theorem applied to the restriction of $\pi_x$ to $S$, in a small neighborhood of any $x\in S$, the manifold $S$ can be represented as the graph of a smooth map $F_x\colon T_xS \to N_x S$. In particular, it follows that in this neighborhood one has $\pip_x = F_x \circ \pi_x$. Moreover, $F_x(0) = 0$, and since the graph of $F_x$ is tangent to $T_xS$ at the origin, the derivative of $F_x$ at $x$ vanishes. Hence, the Taylor expansion of $F_x$ at $0$ starts with the second-order term, and consequently, there exists $C>0$ such that for small $v$ 
\begin{equation}
\label{eq:local-parametrization-norm-estimate}
    \norm{F_x(v)}_{N_xS} \leq C\norm{v}_{T_xS}^2.
\end{equation}

We will precede the statement of the core estimate (Theorem~\ref{thm:core-estimate}) with five lemmas required for its proof. Denote by $B(x,r)$ the ball of radius $r$ in $\R^D$, centered at $x$. 

\begin{lemma}
\label{lem:ball-containment}
    Let $x\in S$. For sufficiently small $\delta$ the projection $\pi_x(S\cap B(x,\delta^{1/2}))$ contains the ball $B(x,\delta) \cap T_xS$. 
\end{lemma}

\begin{proof}
\label{lem:ball-containment-proof}
    Assume that $\delta$ is sufficiently small for $F_x$ to be defined on $B(x,\delta) \cap T_xS$. Let $v\in B(x,\delta) \cap T_xS$. Under our identifications, $y=(v,F_x(v)) \in T_xS \oplus N_xS$ is a point of $S$ such that $v=\pi_x(y)$. Moreover, by eq.~\eqref{eq:local-parametrization-norm-estimate}, for sufficiently small $\delta$
    \begin{equation}
        \norm{v}^2_{T_xS} + \norm{F_x(v)}^2_{N_xS} \leq \delta^2(1+ C\delta^2) < \delta,
    \end{equation}
    so $y\in B(x,\delta^{1/2})$, and $v \in \pi_x(S\cap B(x,\delta^{1/2}))$.
\end{proof}

\begin{lemma}
\label{lem:tangent-gaussian-estimate}
    For $x\in S$ and sufficiently small $\delta$, the estimate
    \begin{equation}
        \int_{S\cap B(x,\delta^{1/2})} \phi^{T_xS}_\delta(\pi_x(y))\,dp_S(y) \asymp 1
    \end{equation}
    holds uniformly in $\delta$.
\end{lemma}

\begin{proof}
    Denote $B=S\cap B(x,\delta^{1/2})$. Integrating by substitution, we obtain
    \begin{equation}
        \int_B \phi^{T_xS}_\delta(\pi_x(y))\,dp_S = \int_{\pi_x(B)}\phi^{T_xS}_\delta(v) \,d(\pi_x)_*p_S,
    \end{equation}
    where the pushforward $(\pi_x)_*p_S$ is a smooth measure on $T_xS$. Hence, for sufficiently small $\delta$
    \begin{equation}
    \label{eq:approximating-integral-of-tangential-normal-distribution}
        \int_{\pi_x(B)} \phi^{T_xS}_\delta(v) \,d(\pi_x)_*p_S(v) \asymp \int_{\pi_x(B)} \phi^{T_xS}_\delta(v) \,dv
    \end{equation}
    uniformly in $\delta$. The integral on the right is at most $1$, and simultaneously, by Lemma~\ref{lem:ball-containment} we have
    \begin{equation}
        \int_{\pi_x(B)} \phi^{T_xS}_\delta(v) \,dv \geq \int_{B(x,\delta)\cap T_xS} \phi^{T_xS}_\delta(v) \,dv.
    \end{equation}
    The last integral is the probability that a normal random variable falls within one standard deviation from the mean, which is a constant independent of $\delta$.
\end{proof}

\begin{lemma}
\label{lem:normal-gaussian-estimate}
    For sufficiently small $\delta$ and $y\in S\cap B(x, \delta^{1/2})$, where $x\in S$, the estimate $\phi^{N_xS}_\delta(\pip_x(y)) \asymp \delta^{d-D}$ holds uniformly in $\delta$ and $y$.
\end{lemma}

\begin{proof}
    By eq.~\eqref{eq:dilated-gaussian}, 
    \begin{equation}
        \phi^{N_xS}_\delta(\pip_x(y)) = \delta^{d-D}\phi^{N_xS}(\delta^{-1}\pip_x(y)).
    \end{equation}
    Since $\pi_x$ is a contraction, we have $\norm{\pi_x(y)} \leq \delta^{1/2}$. It follows from eq.~\eqref{eq:local-parametrization-norm-estimate}, that for sufficiently small $\delta$
    \begin{equation}
        \norm{\pip_x(y)} = \norm{F_x(\pi_x(y))} \leq C\delta
    \end{equation}
    for some $C>0$. Therefore $\delta^{-1}\pip_x(y)$ lies inside a fixed ball independent of $\delta$, and in consequence
    \begin{equation}
        \phi^{N_xS}(\delta^{-1}\pip_x(y)) \asymp 1
    \end{equation}
    uniformly in $\delta$, concluding the proof.
\end{proof}

\begin{lemma}
\label{lem:gaussian-estimate-on-ball}
    For $x\in S$ and sufficiently small $\delta$,
    \begin{equation}
        \int_{S\cap B(x,\delta^{1/2})} \phi^D_\delta(x-y)\, dp_S(y) \asymp \delta^{d-D}
    \end{equation}
    uniformly in $\delta$.
\end{lemma}

\begin{proof}
    Denote $B=S\cap B(x,\delta^{1/2})$. By eq.~\eqref{eq:gaussian-decomposition} and Lemma~\ref{lem:normal-gaussian-estimate}, for sufficiently small $\delta$ and $y\in B$, we have
    \begin{equation}
        \phi^D_\delta(x-y) \asymp \delta^{d-D} \phi^{T_xS}_\delta(\pi_x(y)) 
    \end{equation}
    uniformly in $\delta$. It follows that the original integral can be estimated as
    \begin{equation}
        \int_B \phi^D_\delta(x-y)\, dp_S \asymp \delta^{d-D} \int_B \phi^{T_xS}_\delta(\pi_x(y)) \, dp_S.
    \end{equation}
    The proof concludes by applying Lemma~\ref{lem:tangent-gaussian-estimate} to the last integral. 
\end{proof}

\begin{lemma}
\label{lem:gaussian-estimate-outside-ball}
    For every $x\in S$ 
    \begin{equation}
       \lim_{\delta\to 0^+}\,\int_{S\setminus B(x,\delta^{1/2})} \phi^D_\delta(x-y)\, dp_S(y) = 0.
    \end{equation}
\end{lemma}

\begin{proof}
    Observe that if $\norm{v} \geq \delta^{1/2}$, we have
    \begin{equation}
        \phi^D_\delta(v) \asymp \delta^{-D} \exp \left(-\frac{\norm{v}^2}{2\delta^2}\right) \leq \delta^{-D}\exp\left(-\frac{1}{2\delta}\right)
    \end{equation}
    uniformly in $v$. This bound on the integrand converges to $0$ as $\delta\to 0$, and the measure $p_S$ is finite, proving the convergence of the considered integral.
\end{proof}

\begin{theorem}[The core estimate]
\label{thm:core-estimate}
    Assume that $S \subset \R^D$ is a connected $d$-dimensional submanifold endowed with a smooth probability measure $p_S$. Let $\rho_\delta$ be the density of $p_S * \N(0, \delta^2 I)$ on $\R^D$. Then for $x\in S$ and sufficiently small $\delta$, we have
    \begin{equation}
    \label{eq:core-estimate}
        \log\rho_\delta(x) = (d-D)\log\delta + O(1).
    \end{equation}
\end{theorem}

\begin{proof}
    Since $\delta^{d-D}\geq 1$ for $\delta \leq 1$, given sufficiently small $\delta$, from Lemma~\ref{lem:gaussian-estimate-outside-ball} we get
    \begin{equation}
        \int_{S\setminus B(x,\delta^{1/2})} \phi^D_\delta(x-y)\, dp_S(y) < 
        \delta^{d-D}.
    \end{equation}
    By combining this with eq.~\eqref{eq:mollified-density} and Lemma~\ref{lem:gaussian-estimate-on-ball}, we obtain $\rho_\delta(x) \asymp \delta^{d-D}$, which yields the desired 
    estimate after taking $\log$.
\end{proof}

\subsection{The LIDL algorithm} 

\begin{algorithm}[b]
\caption{LIDL algorithm}
\label{alg:lidl}
\begin{algorithmic}
\REQUIRE
    $X\subset \R^D$; %
    $x_1,\dots, x_m \in\R^D$;  %
    $\delta_1, \dots, \delta_n \in \R^+$;  %

    \FOR{$j=1$ to $n$}
        \STATE $\DD_j \gets$ $\DD$ perturbed with $\N(0, \delta_j^2 \I_D)$ 
        \STATE Fit the density model $\hat\rho_j$ to $\DD_j$
    \ENDFOR
    \FOR{$i=1$ \TO $m$}
        \FOR{$j=1$ \TO $n$}
            \STATE $\xi_j \gets \log\delta_j$
            \STATE $\eta_j \gets \log\hat\rho_j(x_i)$
        \ENDFOR
        \STATE $\beta \gets$  regression coefficient for a set of $n$ points $(\xi_j, \eta_j)$
        \STATE $\hat{d}_i \gets D + \beta$
    \ENDFOR
    \STATE \textbf{return} $(\hat{d}_1, \dots, \hat{d}_m)$
\end{algorithmic}
\end{algorithm}

Now, let us consider how to use the core estimate derived above in practice. 
The core requirement of LIDL is access to the approximate densities $\rho_\delta(x)$, which we have to obtain by fitting a density estimator on the data points from the dataset perturbed with a normally distributed noise of an appropriate magnitude $\delta$. 
Luckily, these days there exist density estimators which scale to data even as high-dimensional as images.
For the purpose of empirical evaluation of our method, in this work we use three models from the family of NF, however we emphasize that our method could use absolutely any density estimation method.
A viable alternative could be, for example, using diffusion models \citep{song2021maximum}, which are likely to lead to further improved accuracy of LIDL estimates.

Given a dataset $\DD\subset\R^D$, and a point $x\in\R^D$, at which we want to estimate LID (usually $x\in D$, as we want to take a point from the image of the data manifold in $\R^D$), we proceed as follows. First, we choose $n>1$ values $\delta_1, \dots, \delta_n$ of perturbation magnitude. We discuss how to choose $\delta$ in the following section. Then, we fit $n$ probability densities $\hat\rho_i$, which will be our approximations of $\rho_{\delta_i}$. Having estimated the densities $\hat\rho_i$, we consider the sequence of points of the form $(\log \delta_i, \log\hat\rho_i(x))$. Using linear regression to fit eq.~\eqref{eq:core-estimate}, we get an estimate $\beta$ for $d-D$, from which we obtain $\hat{d} = D + \beta$, an estimate for $d$. To estimate LID for multiple points, we can fit the densities once, and then loop over the points. 
Full algorithm is presented in Algorithm \ref{alg:lidl}.

It is worth noting that our method fits nicely into the LID estimation framework presented in \cite{amsaleg2019intrinsic}. Roughly speaking, it is based on two observations. Firstly, the dimension of an Euclidean space can be recovered from the degree of the polynomial growth rate of its ball volume as a function of its radius. Secondly, this idea can be applied to discrete datasets by replacing the notion of ball volume with the likelihood function of finding a point of the dataset within a given distance from a fixed base point. 

In our notation, this likelihood function is $r\mapsto p_S(B(x,r))$, where $x$ is the base point. With reasonable assumptions on the measure $p_S$, it can be shown that for small $r$ this function behaves like a polynomial of degree $d$, so the LID value we are estimating is the same as what is defined in \cite{amsaleg2019intrinsic}, which can be consulted for more details.

\input{scale}

%% file: scale.tex
\begin{figure}[b]
    \centering
    \includegraphics[width=0.4\textwidth]{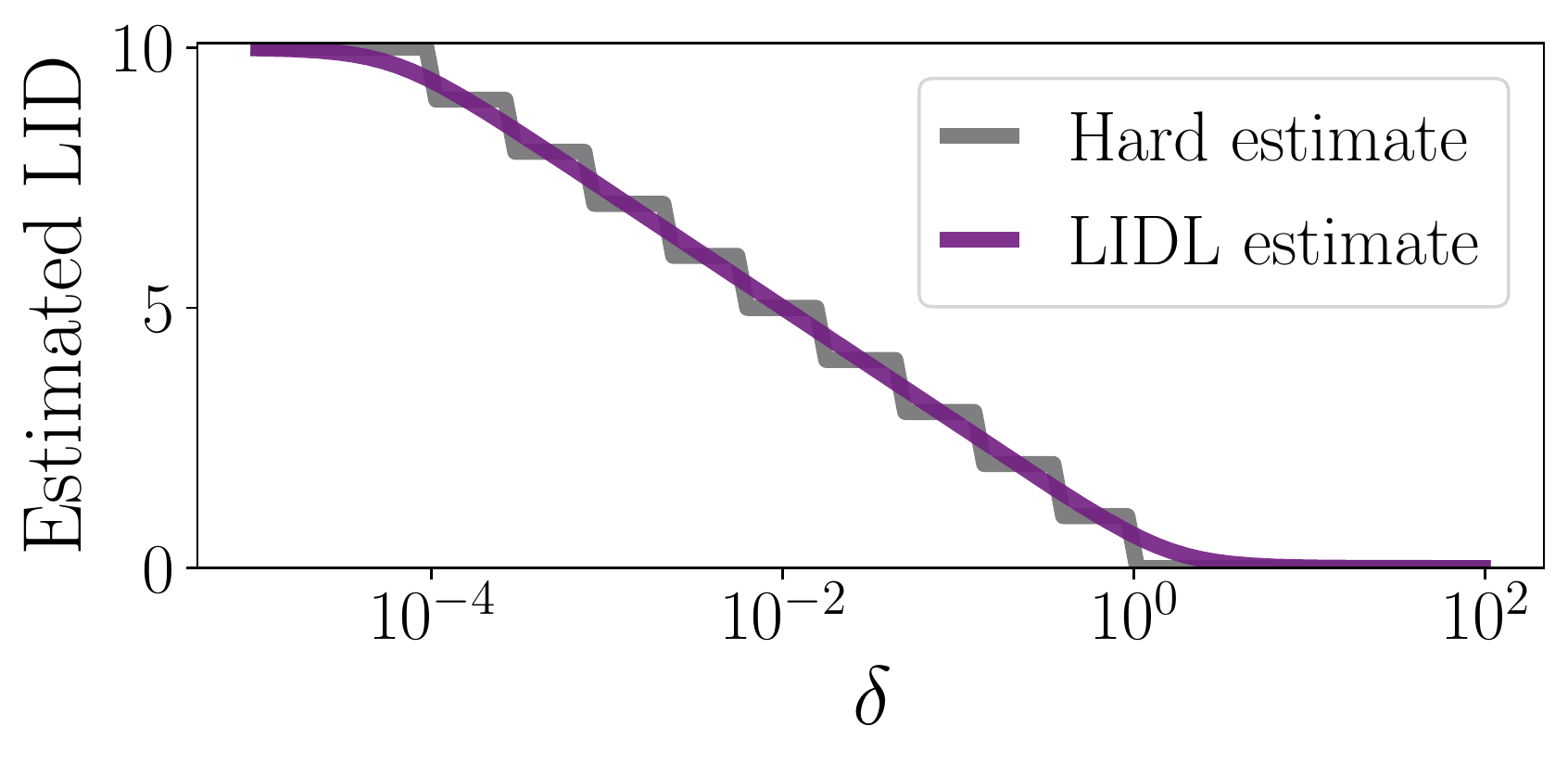}
    \caption{LIDL and hard estimates for different values of $\delta$ for 10D non-isotripic Gaussian. Notice how LIDL ignores dimensions smaller than $\delta$, as predicted theoretically.}
    \label{fig:step-delta-exp}
\end{figure}

\subsection{Viewing $\delta$ as a scale parameter}
\label{sec:scale}

In the previous section, we glossed over the fact that we have to choose the values of $\delta$.
From Sec.~\ref{sec:core-estimate} we know that the core estimate is exact for infinitesimally small $\delta$, so the first pressure is for the $\delta$ to be small, possibly as small as the numerical precision allows. 

\begin{figure}[t]
    \centering
    \includegraphics[width=0.4\textwidth]{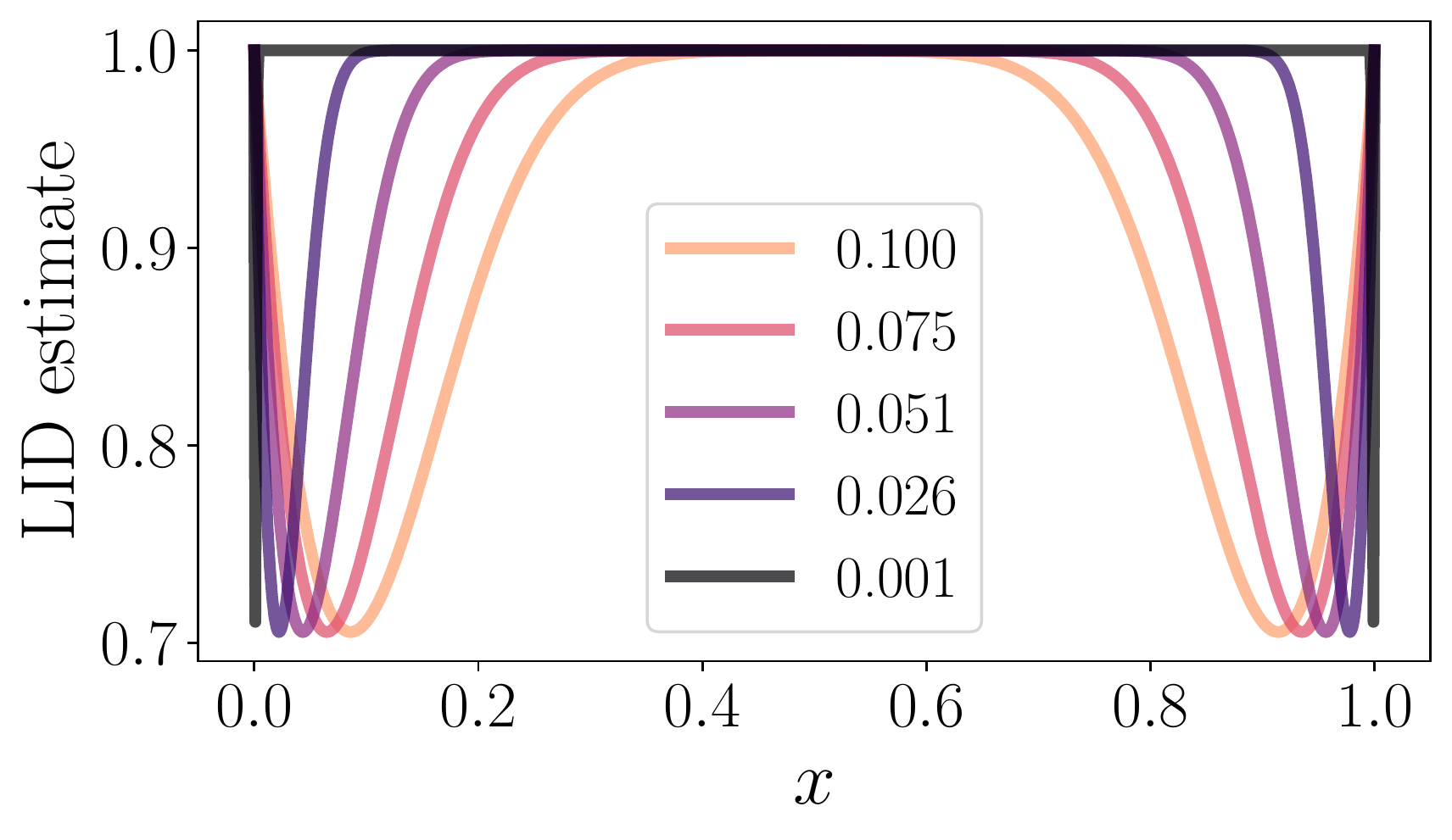}
    \caption{LIDL estimates for points from $\U(0,1)$ for different values of $\delta$ marked with different colors, as explained in the legend of the plot.}
    \label{fig:evaluation-uniform}
\end{figure}

However, $\delta$ can also be viewed as a length-scale parameter that allows users to choose a certain minimum `thickness' to be considered, such that dimensions `thinner' than the threshold will be ignored.
Consider an illustrative example:
suppose that the probability distribution $p_S$ is concentrated in a tubular neighborhood of another submanifold $S'$ of dimension $d'<d$. 
In this case, it can be approximated by a probability distribution $p_{S'}$ supported on $S'$.

Now, if this approximation is `good', in the sense that the thickness of the considered neighborhood is much smaller than the values of $\delta$ used, then intuitively the LIDL estimate should actually reflect the dimension $d'$ of the submanifold $S'$ instead of the true dimension $d$.

Continuing this example, 
we present the described behavior empirically. 
Let $S=\R^D$, and $p_S = \N(0, \Sigma)$, where $\Sigma$ is a diagonal matrix with entries $\sigma_1^2 \geq \sigma_2^2 \geq\dots\geq \sigma_D^2$. 
In Fig.~\ref{fig:step-delta-exp}, we plot the LIDL estimates for case $D=10$ and for $\sigma$ equally distributed on the logarithmic scale. We also plot the \emph{hard estimate} which is simply counting the number of entries in $\sigma$ that are larger than $\delta$. The LIDL estimate follows the hard estimate, and this behavior is predicted theoretically in Appendix~\ref{sec:appendix-lidl-for-normal-distribution}. 
We further investigate the role of the scale parameter in the case of imperfect density estimates in Sec.~\ref{sec:operating-range}.

As mentioned earlier, 
having an explicit length-scale parameter 
can be considered LIDL's feature as compared to other methods.
Is allows the user to easily set an operating scale such that to ignore certain amplitude of noise in the original data, e.g.\ the observation noise if we are able to estimate its magnitude apriori. 
The empirically observed rule of thumb is to take at least $\delta \gtrsim 10\sigma$, where $\sigma$ is standard deviation of the noise to be ignored. 

Setting such operating scale characteristic can be difficult in many other non-parametric algorithms that calculate statistics based on nearest neighbors. 
In those approach, there is either of the two natural scale parameters: number of nearest neighbours $k$ or radius $r$ around the point where we search for neighbours. 

When using $k$, our effective operating range depends on a combination of local density and the total number of samples used to run the algorithm. 
Using $r$ allows to set an operating range.
In this case, however, we expose ourselves to the risk of having not enough samples to estimate the local density. 
Most implementations of those methods set a default $k$.

%% file: empirical.tex
\section{Empirical Behavior of the Proposed Method
}
\label{sec:empirical-evaluation}

In this section we examine the behaviour of our method when confronted with certain isolated difficulties. 
Instead of relying on a computed approximation $\hat{\rho}_\delta$, we assume we are given the actual perturbed density $\rho_\delta$ explicitly or we compute it through numerical integration. This ensures that any error observed during this analysis is caused directly by our LIDL method and not the density estimator. However, it comes at a price of restricting us to relatively simple examples where we can efficiently compute $\rho_\delta$.

\subsection{Uniform density on an interval}

\begin{figure}[t]
    \centering
    \includegraphics[width=0.4\textwidth]{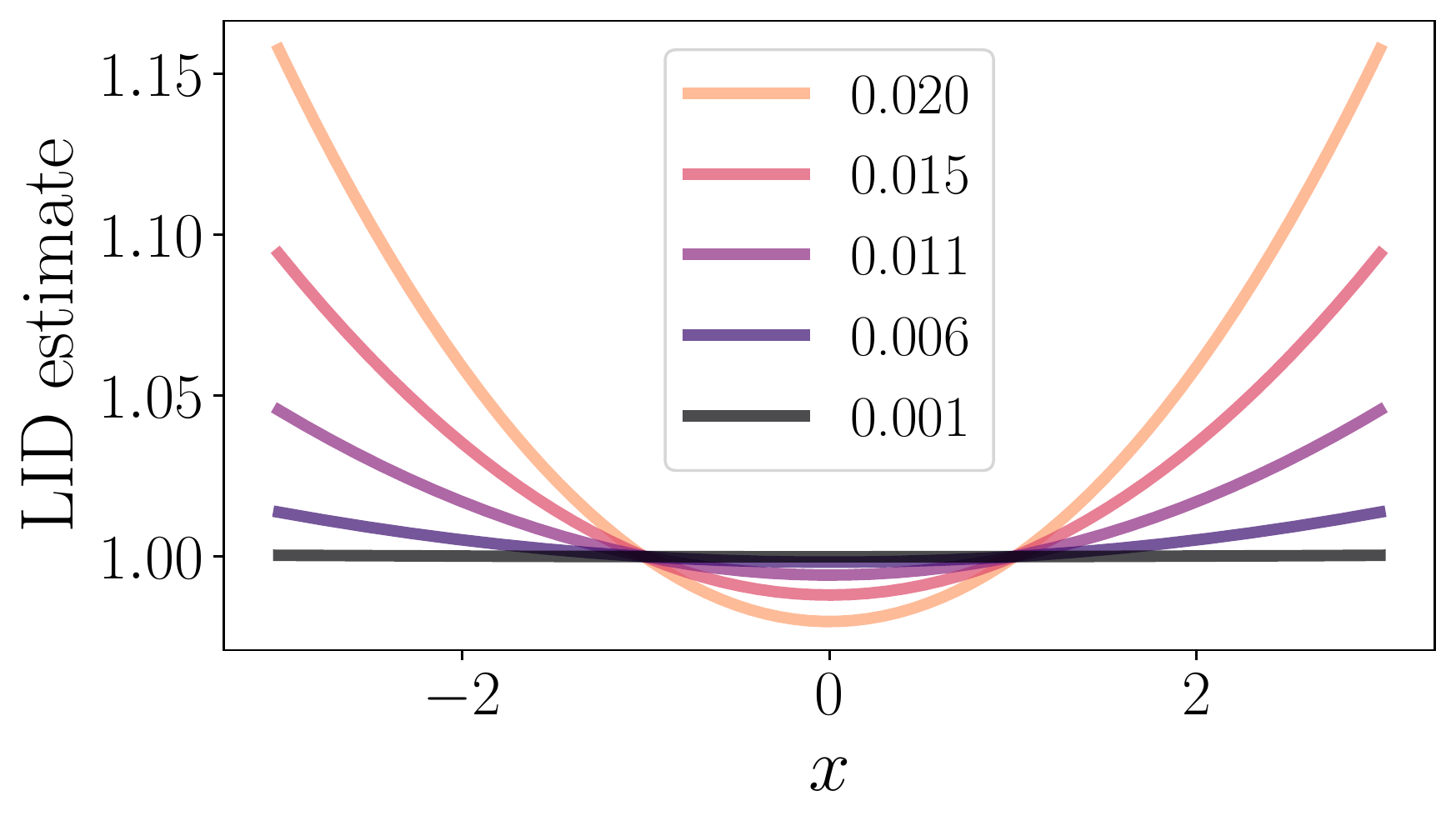}
    \caption{LIDL estimates for points from $\N(0,1)$ for different values of $\delta$ marked with different colors, as explained in the legend of the plot.}
    \label{fig:evaluation-normal}
\end{figure}

We assume, that in the neighborhood of $x$ the density is bounded from below by a positive constant. But for some real-world cases, this assumption is not fulfilled. To investigate how LIDL behaves in this case we ran it on $\U(0,1)$. It can be seen as a distribution on the real line, whose density vanishes outside $[0,1]$ interval, violating this assumption. Alternatively, in the vicinity of the interval endpoints, the size of the neighborhood admitting the parametrization required for the proof of the core estimate decreases to $0$.

We analytically calculated the convolution of $\U(0,1)$ with $\N(0,\delta^2)$ and used it to estimate LID at 1000 points between 0 and 1. 
We used just two points for linear regression, corresponding to $\delta_1=\delta$ and $\delta_2=1.05\delta$. The estimates for different values of $\delta$ are plotted in Fig.~\ref{fig:evaluation-uniform}. We can see that an error is introduced near the boundary as expected. In this case, its maximum value does not depend on the value of $\delta$, and points affected by this problem lie in the part closer than $\sim4\delta$ to the endpoints of the distribution support.

\subsection{Normal distribution on a line}

In this example, we study the LID estimates at different points of a line embedded in $\R^D$.
In Fig.~\ref{fig:evaluation-normal} we can see the estimates computed as per the previous example, for a few values of $\delta$. 
At first glance, it is worrying that the error seems to explode with distance from the mean of the distribution.
In Appendix~\ref{sec:appendix-lidl-for-normal-distribution}, we show that the error is quadratic in this distance, and, reassuringly, that its expected value over the whole distribution can be controlled.
The reason for this behavior can be traced back to the proof of Lemma~\ref{lem:tangent-gaussian-estimate} (more specifically eq.~\eqref{eq:approximating-integral-of-tangential-normal-distribution} in the appendix), which depends on the positive constant locally bounding the density from below. 
In our example, the density decreases as $e^{-t^2/2}$, which produces the quadratic error term (the final error is bounded by $\sum_i\abs{\log C_i}$, where $C_i$ are the multiplicative estimate constants appearing in all the steps of the proof). 

\subsection{Uniform density on a curved manifold}

\begin{figure}[tb]
    \centering
    \includegraphics[width=0.4\textwidth]{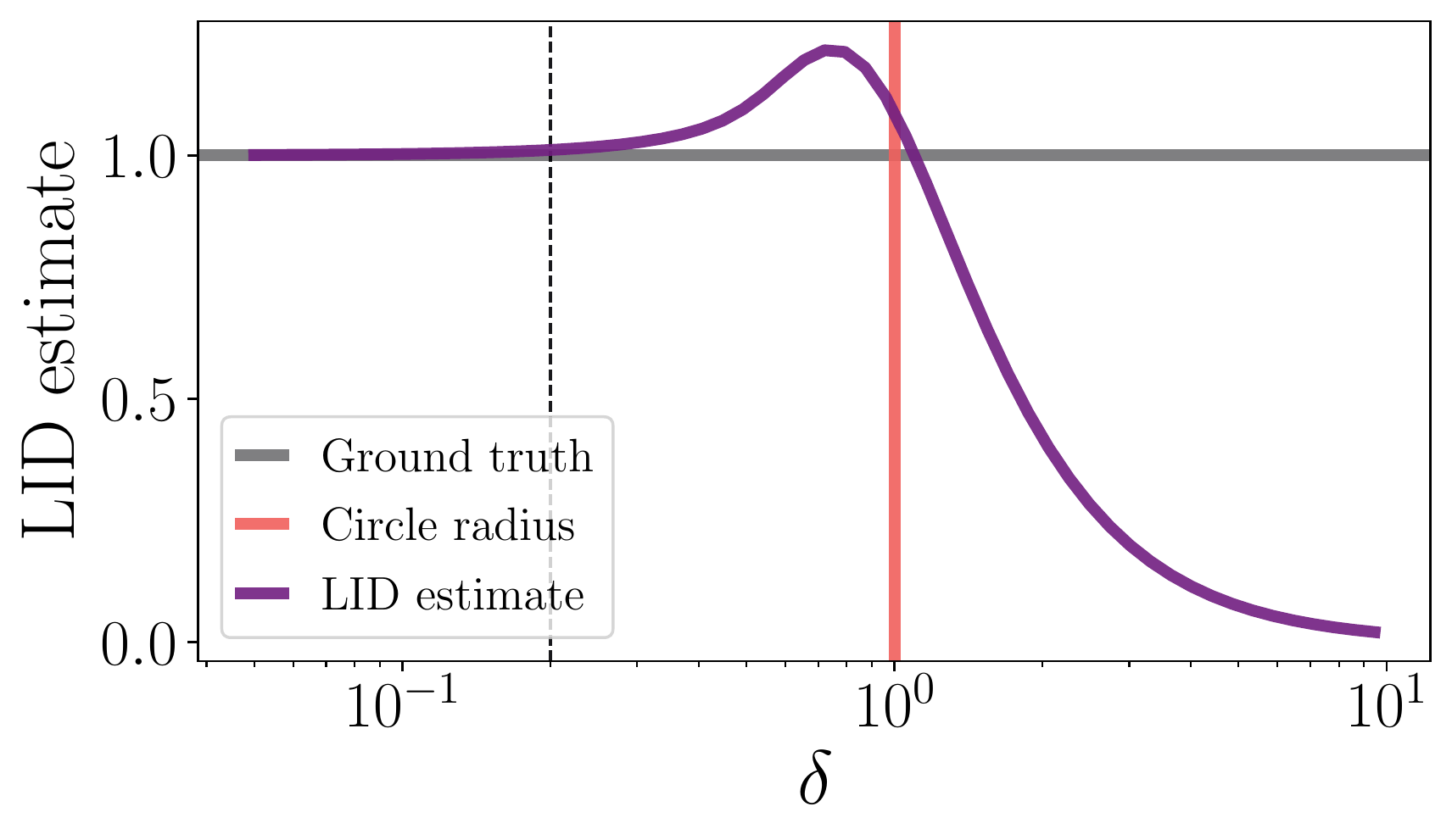}
    \caption{LIDL estimate as a function of $\delta$ for a uniform density on a unit circle. Vertical line at $\delta=0.2$.}
    \label{fig:evaluation-curvature}
\end{figure}

The LIDL estimate is affected by the curvature of the manifold, which manifests in the constant $C$ appearing in eq.~\eqref{eq:local-parametrization-norm-estimate}, subsequently used in the proofs of Lemmas~\ref{lem:ball-containment} and~\ref{lem:normal-gaussian-estimate}.
To see empirically how  the curvature influences the LIDL estimate, we numerically computed the convolution of the uniform density on the unit circle embedded in $\R^2$ with the noise distribution $\N(0,\delta^2I)$ for 2 values of $\delta$ similarly as in the previous examples. We calculated LIDL for the range of $\delta \in (0.05, 10)$. 
We plot the estimate dependence on $\delta$ in Fig.~\ref{fig:evaluation-curvature}. 

We can see that for $\delta \lesssim 0.2$ the estimate error is relatively small. After the positive bias for $\delta < 1$  we can observe a monotonic drop in the estimate until it reaches nearly 0. This is by the effect described in Section~\ref{sec:scale} where LIDL was observed to ignore the directions in which the standard deviations were lower than $\delta$.

\subsection{Manifolds with neighboring components}

In a real-world setting, it is possible for some connected components of the data manifold $S$ to be close to each other in the observable data space $\R^D$, especially when some features in the dataset have discrete distribution (e.g.\ height and sex in a medical dataset). In those settings, for values of $\delta$ comparable to the distance between the components, additional bias may be introduced to the estimate. To investigate this we ran an experiment similar to the previous example, but with a uniform distribution supported on the union of two long parallel segments. We then calculated LIDL estimates for the midpoints of those segments, to minimize the error caused by proximity to the boundary. We present the results in Fig.~\ref{fig:evaluation-close}. We can see positive bias in LIDL estimate appearing as $\delta$ is close to the distance between the segments, while for $\delta$ much larger than this distance, LIDL seems to view those two segments as a single line.

\subsection{Impact of linear regression on LIDL estimate}
Because our estimate depends on linear regression algorithm in order to estimate $\beta$, it may suffer from the same issues as any regression coefficient estimation algorithm \cite{li1985robust}, so in the future, more robust algorithm for linear regression estimation may be considered. Because we estimate only the rate of change, and not the constant from linear regression equation, LIDL is prone to biased log-likelihood estimates, and noise added to log-likelihood estimate only affects the variance of the estimate.

\begin{figure}[tb]
    \centering
    \includegraphics[width=0.4\textwidth]{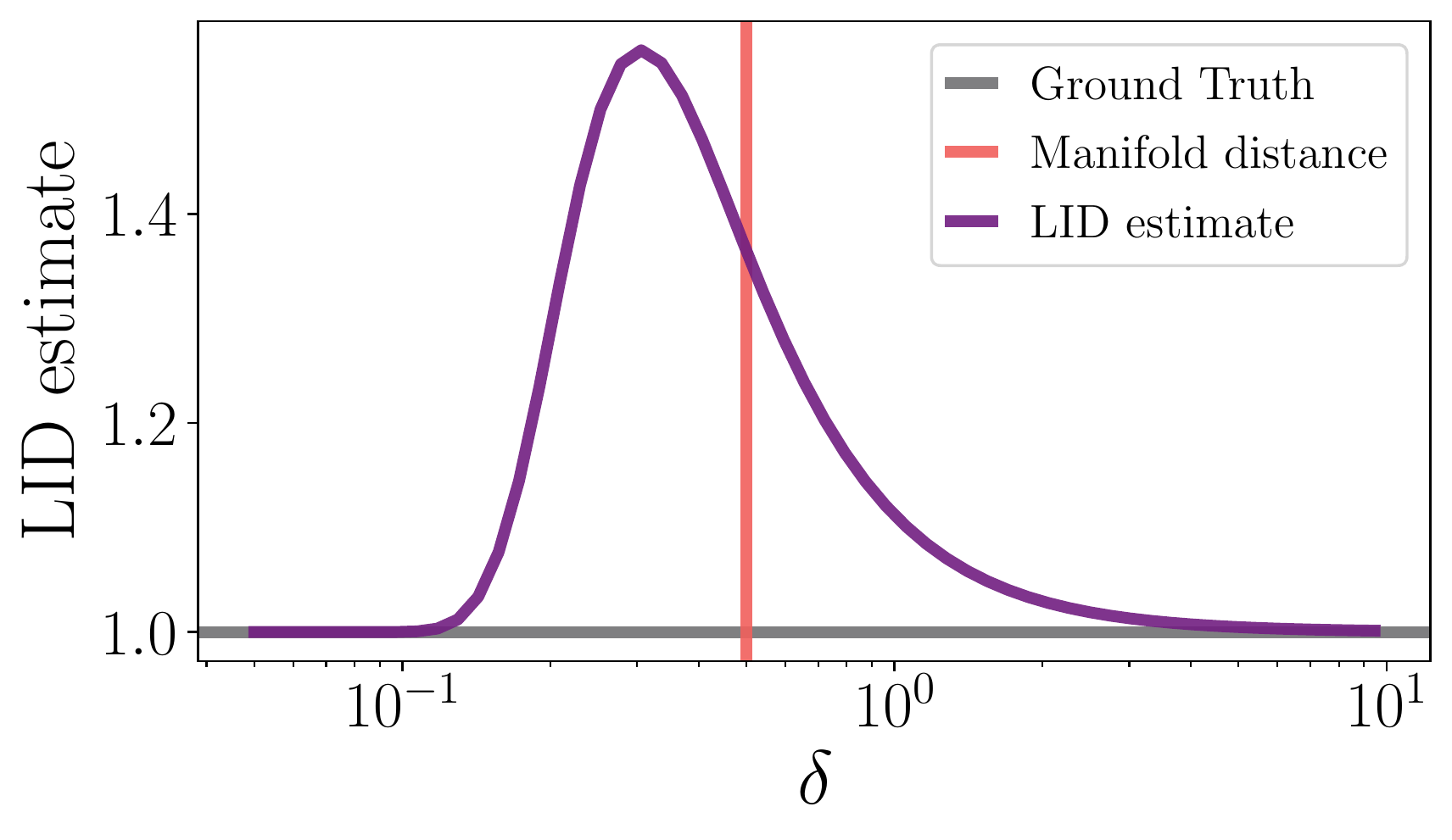}
    \caption{LIDL estimate as a function of $\delta$ for 2 long 1-dimensional manifolds parallel to each other.}
    \label{fig:evaluation-close}
\end{figure}

\subsection{Synthetic datasets}
\label{sec:other-perfect-datasets}
We ran evaluations of LIDL with density estimates computed using numerical integration on Swiss roll, uniform distribution on a helix, and Gaussians from 10 up to 4000 dimensions. We got almost exact estimates with mean absolute error (MAE) lower than $10^{-4}$ for every dataset. More information about these results is included in Appendix~\ref{sec:exp-details}.

%% file: related.tex
\section{Related Work}

ID estimation methods can be divided into two broad categories: global and local \citep{camastra2016intrinsic}. Global methods aim to give a single estimate of the dimensionality of the entire dataset, which however discards the nuanced manifold structure when the data lies on a union of different manifolds (which is often the case for real-world datasets). 

On the contrary, the local methods \citep{carter2009local,kleindessner2015dimensionality,levina2004maximum,hino2017local,camastra2016intrinsic, rozza2012novel,CERUTI20142569,camastra2002} try to estimate the local ID of the data manifold at an arbitrary point.
This approach gives more insight into the nature of the dataset, and provides more options to summarize the dimensionality of the manifold than the global perspective.
A detailed overview of the methods used for global and local ID estimation is provided by \citet{camastra2016intrinsic}, and for a good review of the local ID estimation methods we refer the reader to \citet{johnsson2014low}. 
We list all the algorithms we compare to in Table~\ref{tab:skdim} in the appendices.

%% file: experiments.tex
\section{Experiments}
\label{sec:exp}

\begin{figure}
    \centering
    \includegraphics[width=0.45\textwidth]{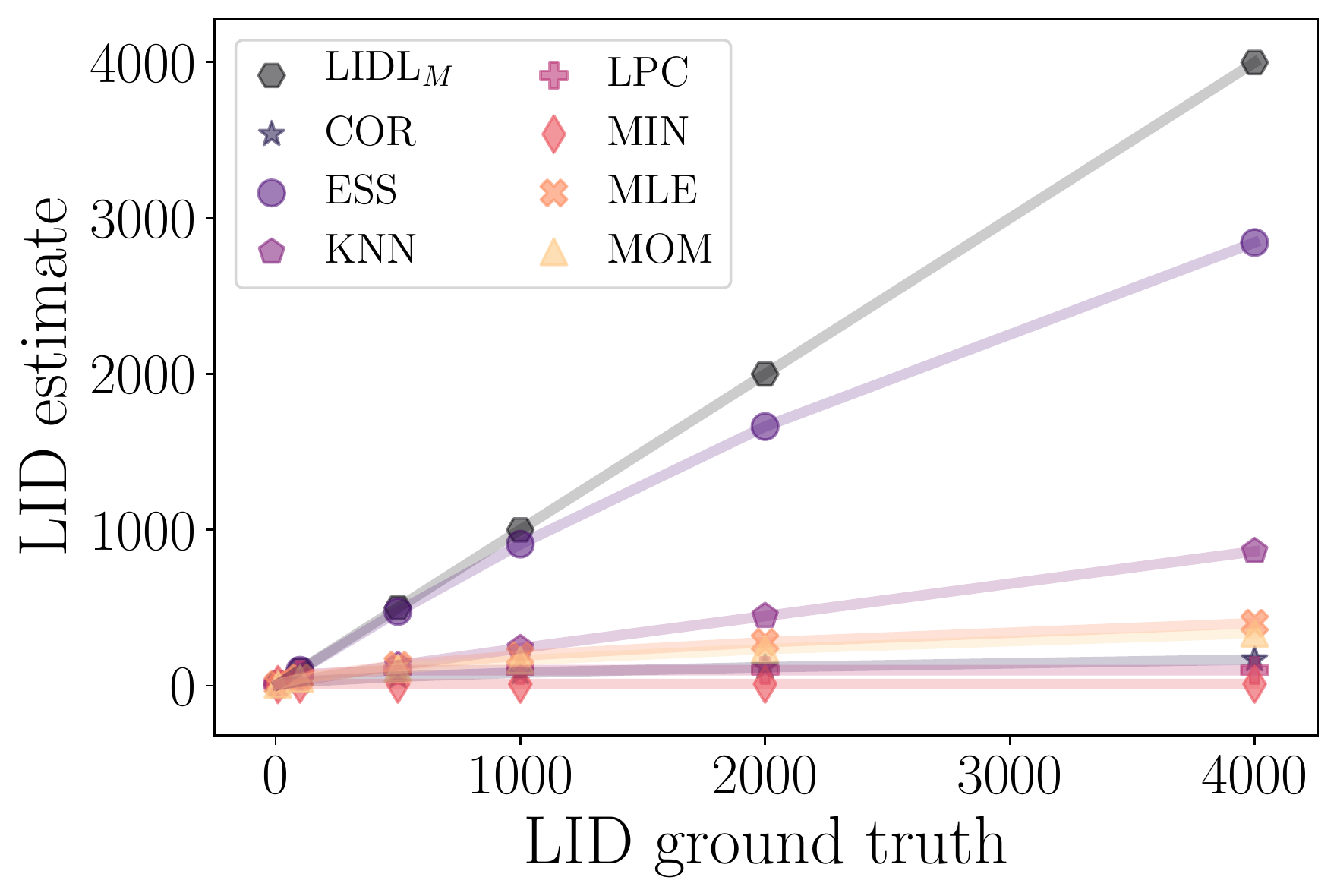}
    \caption{LID estimates for $d$ dimensional uniform distribution on a hypercube. More results and abbreviation explanations can be found in Tables~\ref{tab:comparison},~\ref{tab:comparison-mae}~and~\ref{tab:skdim} in Appendix~\ref{sec:exp-details}.  The dimensionality $d$ of the distribution is plotted on the horizontal axis and the estimates for different algorithms on the vertical axis.}
    \label{fig:uniform-scalability}
\end{figure}

In this section, we compare LIDL with other algorithms, investigate its behavior with imperfect density estimators, and run it on real-world datasets. Details of training procedure can be found in Appendix~\ref{sec:exp-details} and in Sec.~\ref{sec:exp-reducing} we describe how to reduce an error of our estimate.

\subsection{Comparison on synthetic datasets}

We collated LIDL with other LID estimation algorithms from \texttt{scikit-dimension} Python library \citep{e23101368}, which covers all of the important algorithms for LID estimation, and compared them in three different aspects: %
\emph{1}.\ Scalability, 
\emph{2}.\ Multidimensional and curved manifolds,
\mbox{\emph{3}.\ Multiscale manifolds}.

We excluded FisherS and DANCo algorithms because they do not scale well to higher-dimensional settings. FisherS suffered from memory problems on medium datasets, and DANCo had unfeasibly long runtimes (multiple weeks) on the thousand-dimensional datasets.
According to the convention in the field, we choose to make comparison only on synthetic datasets, because we have ground truth for them.

\begin{figure}
    \centering
    \includegraphics[width=0.44\textwidth]{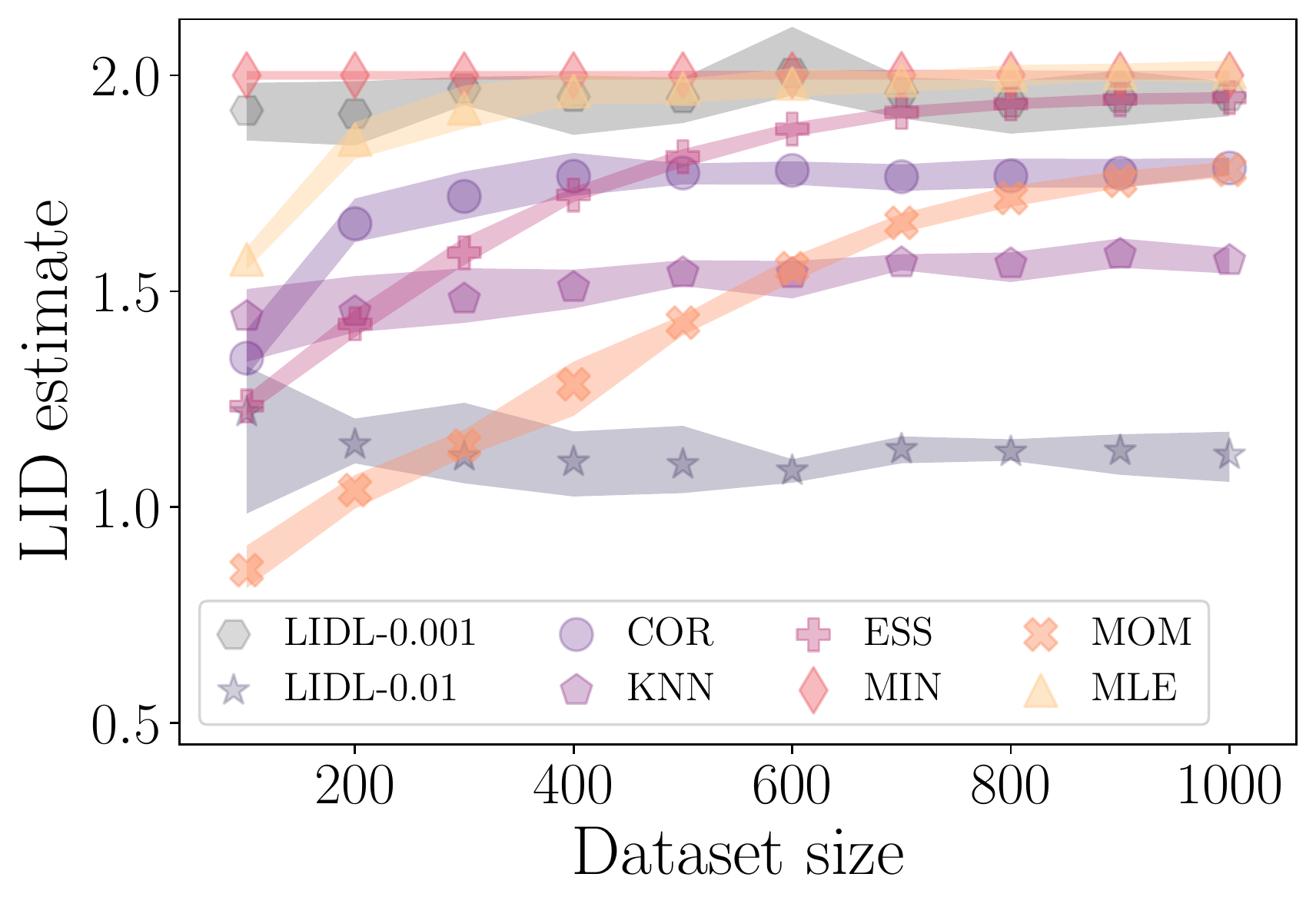}
    \caption{LID estimates for uniform distribution on a rectangle with edge lengths equal to 0.1 and 0.01. The size of the dataset is plotted on the horizontal axis and the estimate (with respective 95\% confidence intervals) on the vertical axis. For most algorithms (except LIDL, KNN and MIN), we can see a disturbing phenomenon: the estimate depends on the sample size. LIDL-$\delta$ stands for LIDL with MAF density estimator and scale parameter $\delta$.Other abbreviations are explained in appendix in Table~\ref{tab:skdim}.}
    \label{fig:multiscale-comparison}
\end{figure}

\paragraph{Scalability}

To test scalability we ran all algorithms on standard multidimensional uniform and normal distributions up to 4K dimensions. Detailed results of the comparison are gathered in Appendix~\ref{sec:exp-details} in Tables~\ref{tab:comparison} and \ref{tab:comparison-mae} (starting from the 7-th row). Each dataset consisted of 10K data points and each algorithm was run 5 times on different samples from the distribution. For each run, we calculated differences between true LID and estimate and averaged it over 5 runs. Then we divided the result by the average manifold dimensionality for each dataset, getting a relative bias of each algorithm. In subsequent tables, we report relative MAE and estimate standard deviation for the same procedure. 
From those tables, we can clearly see, that although in many cases LIDL does not have the lowest error and bias, for almost all datasets the results are in the $\pm5\%$ range. Other algorithms fail to accurately estimate dimensions exceeding 100. One exception is ESS, which stands out from the rest but remains inferior to LIDL. We plot LID estimates for some of the algorithms (we omitted few for the sake of clarity) for multidimensional uniform distributions in Fig.~\ref{fig:uniform-scalability}. All the abbreviations used in the plot are explained in Appendix~\ref{sec:exp-details}.

\begin{figure}[t]
    \centering
    \includegraphics[width=0.49\textwidth]{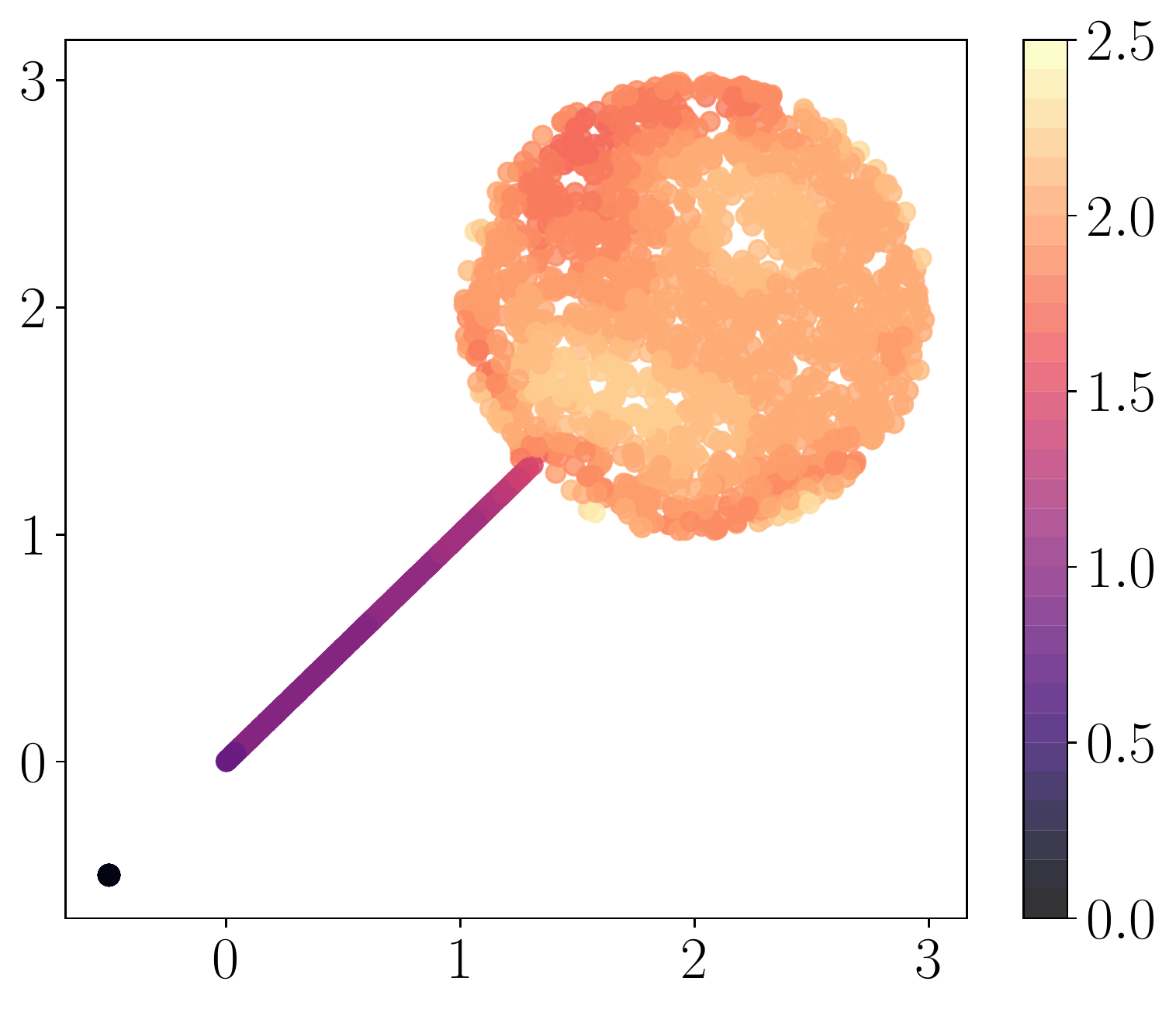}
    \caption{Points from the lollipop benchmark dataset and LIDL (with MAF) estimates for those points.}
    \label{fig:maf-lollipop}
\end{figure}

\paragraph{Multiscale manifolds}
In the Introduction, we postulated that a useful LID estimation algorithm should operate properly on multiscale manifolds. 
In this section, we compare existing LID methods and LIDL on highly non-isotropic datasets. 
We observed that most of the algorithms with the same scale parameters (or those without such parameters, like ESS) give different results for different sizes of the dataset. 
We hypothesize that this may be caused by violating assumptions about the local uniformity of the distribution, but we did not investigate it further. 
Only LIDL, MiND\-ML, DANCo, and KNN give stable estimates for different dataset sizes. 
We plot those results for selected algorithms in Fig.~\ref{fig:multiscale-comparison}. For both scale parameter values, LIDL gives stable estimates for different dataset sizes. 
The rest of the unplotted algorithms also give unstable estimates, and we omitted them only to make the plot more readable.

\paragraph{Curved manifolds and unions of manifolds}

We tested LIDL and other algorithms on some smaller but more complicated manifolds. We used three classical benchmarks from \cite{kleindessner2015dimensionality}: the Swiss roll dataset, uniform density on a helix, and uniform density on a sphere. These datasets lie on a curved manifolds (2-, 1- and 7-dimensional respectively) which may cause difficulties with fitting density estimators. Results of those experiments can be found in rows 4-7 of Tables~\ref{tab:comparison} and \ref{tab:comparison-mae}. The results for LIDL are decent (Relative bias less than 0.05 and relative MAE less than 0.06), but lPCA and ESS gave estimates with relative bias and MAE less than 0.01. For perfect density estimates LIDL gives almost perfect estimates on those datasets, as presented in Sec.~\ref{sec:other-perfect-datasets}.

None of the above datasets however consisted of components of different dimensions, which may be the case for many real-world datasets. We used a \emph{lollipop dataset}, which is composed of 0, 1, and 2-dimensional components. The dataset and its corresponding LIDL estimates are plotted in Fig.~\ref{fig:maf-lollipop}. On the 2- and 1-dimensional parts, many algorithms achieved good results, some even better than LIDL, but the 0-dimensional component, which consisted of replicas of the same point, caused most problems for other algorithms. 

When algorithms tried to estimate LID for this 0-dimensional part, only lPCA and LIDL were able to estimate its dimensionality properly, and almost all other algorithms failed to converge. When we jittered those points a little with $\N(0,10^{-6})$, almost all of the algorithms converged but all of them yielded estimates close to 2. Thanks to the possibility of setting operating scale in LIDL, we could estimate the latter dimension correctly, regardless of noise in the data. Results for each component of the manifold treated separately can be found in the first 3 rows of Tables~\ref{tab:comparison} and \ref{tab:comparison-mae}.

\subsection{Operating range}
\label{sec:operating-range}
As stated in Sec.~\ref{sec:scale}, $\delta$ can be seen as a scale parameter. We introduced some numerical and theoretical results to support this hypothesis, and in this section, we are going to present some experiments investigating this topic. In Fig.~\ref{fig:step-delta-uniform} we present a similar experiment to that from Fig.~\ref{fig:step-delta-exp}, but this time with 4-dimensional uniform density. Results seem quite similar to previous theoretical results. For similar Gaussian distribution, we get an almost identical relation between dimension variance, LIDL estimate and $\delta$.

We also tuned a $\delta$ range on image MNIST and FMNIST to reduce dequantization noise influence on the LIDL estimate. More on those experiments can be found in Appendix~\ref{sec:exp-details}. Although this scale parameter has to be used with care. In one experiment on FMNIST (normalized to have values between -0.5 and 0.5) for values of $\delta > 0.1$ we observed that the whole cluster of darker clothes had been estimated as being 0-dimensional. We present some samples from this cluster in Fig.~\ref{fig:dark}.

\begin{figure}[t]
    \centering
    \includegraphics[width=\linewidth]{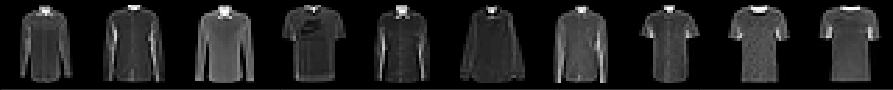}
    \caption{Images from the FMNIST dataset, for which the LID estimate is close to 0. This effect occurred when we used to high $\delta$s for this thin data manifold.}
    \label{fig:dark}
\end{figure}

\subsection{Experiments on image datasets}

\begin{figure}[b]
    \centering
    \newcommand\figwidth{0.16\textwidth}
    \minipage{\figwidth}
    \includegraphics[width=0.99\textwidth]{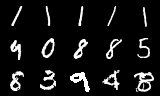}
    \endminipage
    \minipage{\figwidth}
    \includegraphics[width=0.99\textwidth]{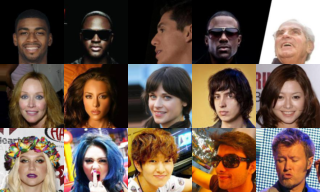}
    \endminipage
    \minipage{\figwidth}
    \includegraphics[width=0.99\textwidth]{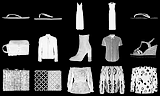}
    \endminipage
    \caption{Samples from different image datasets (MNIST, Celab-A, FMNIST from left to right) presented according to their LIDL estimates (top to bottom). Those results are highly correlated with the  complexity of an image.}
    \label{fig:glow}
\end{figure}

We ran LIDL on MNIST, FMNIST, and Celeb-A ($D$ = 1K, 1K, 12K respectively) datasets using Glow as a density estimator. We sorted those datasets according to LIDL estimates and observed that visually more complex examples have higher LID. Some small, medium, and high dimensional images from those datasets are shown in Fig.~\ref{fig:glow} and Fig.~\ref{fig:mnist-sort}, \ref{fig:fmnist-sort}, \ref{fig:celeba-sort} from Appendix~\ref{sec:exp-details}. In the aforementioned section we plot a distribution of LIDL estimates for different classes from MNIST and FMNIST, and show how the estimate is affected by the dequantization used in NF training.

In Appendix~\ref{sec:performance} we used LIDL to show, that LID negatively correlates with local (per image) accuracy of the classification model for images and that LID is positively correlated with image reconstruction error of VAE. %

%% file: conclusions.tex
\section{Conclusions}
\label{sec:conclusions}

We identified three challenges in LID estimation and explained how the existing methods do satisfy those desiderata. 
To overcome those limitations we introduced an algorithm for LID estimation which relies on powerful neural parametric density estimators, and provided solid theoretical justification for the method. 
Our experiments showed that it can scale to datasets of thousands of dimensions and give accurate estimates on complicated manifolds. 
We investigated its strengths and limitations and showed that LID is connected with local model performance, 
especially in unsupervised learning and classification settings.

There is a number of future research directions stemming from this work. 
The first one is a more theoretically grounded and experimentally tested procedure for choosing $\delta$ in the presence of observation noise, which might be important for practitioners. 
Another one is further investigating the connection of LID estimates and classifier performance: LID estimates could be used in active learning, semi supervised learning or curriculum learning. 

\section{Acknowledgements}
Some experiments were performed using the Entropy cluster at the Institute of Informatics, University of Warsaw, funded by NVIDIA, Intel, the Polish National Science Center grant UMO2017/26/E/ST6/00622 and ERC Starting Grant TOTAL. 
Some experiments were performed using the GUŚLARZ 9000 workstation at the Polish National Institute for Machine Learning.

Jacek Tabor is supported by Foundation for Polish Science (grant no POIR.04.04.00-00-14DE/18-00) carried out within the Team-Net program co-financed by the European Union under the European Regional Development Fund.
Przemysław Spurek is supported by the National Centre of Science (Poland) Grant No. 2019/33/B/ST6/00894.
Adam Goliński is supported by the UK EPSRC CDT in Autonomous Intelligent Machines and Systems.

We wish to thank various people for their contribution to this project: Marek Cygan for overall support during this project and valuable and constructive suggestions on the manuscript; Maciej Dziubiński, Piotr Kozakowski, Dominik Filipak, and Maciej Śliwowski for their valuable and constructive suggestions on the manuscript.

\section{CRediT Author Statement}

\begin{table}[H]
\renewcommand{\arraystretch}{1.4}

\resizebox{\columnwidth}{!}{%
\begin{tabular}{lllllll}
 & \rotatebox{70}{PT} & \rotatebox{70}{RM} & \rotatebox{70}{ŁG} & \rotatebox{70}{PS} & \rotatebox{70}{JT} & \rotatebox{70}{AG} \\ \cline{2-7} 
 
\multicolumn{1}{l|}{Conceptualization} & \multicolumn{1}{l|}{\cellcolor[HTML]{343434}} & \multicolumn{1}{l|}{} & \multicolumn{1}{l|}{} & \multicolumn{1}{l|}{} & \multicolumn{1}{l|}{\cellcolor[HTML]{343434}} & \multicolumn{1}{l|}{\cellcolor[HTML]{343434}} \\ \cline{2-7} 
\multicolumn{1}{l|}{Methodology} & \multicolumn{1}{l|}{\cellcolor[HTML]{343434}{\color[HTML]{656565} }} & \multicolumn{1}{l|}{} & \multicolumn{1}{l|}{\cellcolor[HTML]{343434}} & \multicolumn{1}{l|}{} & \multicolumn{1}{l|}{\cellcolor[HTML]{343434}} & \multicolumn{1}{l|}{\cellcolor[HTML]{343434}} \\ \cline{2-7} 
\multicolumn{1}{l|}{Software} & \multicolumn{1}{l|}{\cellcolor[HTML]{343434}{\color[HTML]{343434} }} & \multicolumn{1}{l|}{\cellcolor[HTML]{343434}{\color[HTML]{656565} }} & \multicolumn{1}{l|}{\cellcolor[HTML]{343434}} & \multicolumn{1}{l|}{} & \multicolumn{1}{l|}{} & \multicolumn{1}{l|}{} \\ \cline{2-7} 
\multicolumn{1}{l|}{Validation} & \multicolumn{1}{l|}{\cellcolor[HTML]{343434}} & \multicolumn{1}{l|}{} & \multicolumn{1}{l|}{} & \multicolumn{1}{l|}{} & \multicolumn{1}{l|}{} & \multicolumn{1}{l|}{} \\ \cline{2-7} 
\multicolumn{1}{l|}{Formal analysis} & \multicolumn{1}{l|}{\cellcolor[HTML]{343434}} & \multicolumn{1}{l|}{} & \multicolumn{1}{l|}{\cellcolor[HTML]{343434}} & \multicolumn{1}{l|}{} & \multicolumn{1}{l|}{} & \multicolumn{1}{l|}{} \\ \cline{2-7} 
\multicolumn{1}{l|}{Investigation} & \multicolumn{1}{l|}{\cellcolor[HTML]{343434}} & \multicolumn{1}{l|}{\cellcolor[HTML]{343434}} & \multicolumn{1}{l|}{} & \multicolumn{1}{l|}{} & \multicolumn{1}{l|}{} & \multicolumn{1}{l|}{} \\ \cline{2-7} 
\multicolumn{1}{l|}{Writing - Original Draft} & \multicolumn{1}{l|}{\cellcolor[HTML]{343434}} & \multicolumn{1}{l|}{} & \multicolumn{1}{l|}{\cellcolor[HTML]{343434}} & \multicolumn{1}{l|}{\cellcolor[HTML]{343434}} & \multicolumn{1}{l|}{} & \multicolumn{1}{l|}{\cellcolor[HTML]{343434}} \\ \cline{2-7} 
\multicolumn{1}{l|}{Writing - Review \& Editing} & \multicolumn{1}{l|}{\cellcolor[HTML]{343434}} & \multicolumn{1}{l|}{} & \multicolumn{1}{l|}{\cellcolor[HTML]{343434}} & \multicolumn{1}{l|}{\cellcolor[HTML]{343434}{\color[HTML]{656565} }} & \multicolumn{1}{l|}{\cellcolor[HTML]{343434}} & \multicolumn{1}{l|}{\cellcolor[HTML]{343434}} \\ \cline{2-7} 
\multicolumn{1}{l|}{Visualization} & \multicolumn{1}{l|}{\cellcolor[HTML]{343434}} & \multicolumn{1}{l|}{} & \multicolumn{1}{l|}{} & \multicolumn{1}{l|}{} & \multicolumn{1}{l|}{} & \multicolumn{1}{l|}{\cellcolor[HTML]{343434}} \\ \cline{2-7} 
\multicolumn{1}{l|}{Supervision} & \multicolumn{1}{l|}{\cellcolor[HTML]{343434}} & \multicolumn{1}{l|}{} & \multicolumn{1}{l|}{} & \multicolumn{1}{l|}{} & \multicolumn{1}{l|}{\cellcolor[HTML]{343434}} & \multicolumn{1}{l|}{\cellcolor[HTML]{343434}} \\ \cline{2-7} 
\multicolumn{1}{l|}{Project administration} & \multicolumn{1}{l|}{\cellcolor[HTML]{343434}} & \multicolumn{1}{l|}{} & \multicolumn{1}{l|}{} & \multicolumn{1}{l|}{} & \multicolumn{1}{l|}{} & \multicolumn{1}{l|}{} \\ \cline{2-7} 
\multicolumn{1}{l|}{Data Curation} & \multicolumn{1}{l|}{} & \multicolumn{1}{l|}{\cellcolor[HTML]{343434}{\color[HTML]{9A0000} }} & \multicolumn{1}{l|}{} & \multicolumn{1}{l|}{} & \multicolumn{1}{l|}{} & \multicolumn{1}{l|}{} \\ \cline{2-7} 
\end{tabular}%
}
\end{table}

\comment {
\textbf{Piotr Tempczyk}: Conceptualization, Methodology, Software, Validation, Formal analysis, Investigation, Writing - Original Draft, Writing - Review \& Editing, Visualization, Supervision, Project administration.
\textbf{Rafał Michaluk}: Software, Investigation, Data Curation.
\textbf{Łukasz Garncarek}: Methodology, Software, Formal analysis, Writing - Original Draft, Writing - Review \& Editing.
\textbf{Przemysław Spurek}: Writing - Original Draft, Writing - Review \& Editing.
\textbf{Jacek Tabor}: Conceptualization, Methodology, Writing - Review \& Editing, Supervision.
\textbf{Adam Goliński}: Conceptualization, Methodology, Writing - Original Draft, Writing - Review \& Editing, Visualization, Supervision.
}

%% file: supplement.tex
\onecolumn
\icmltitle{Supplementary material}
\appendix

\section{Non-connected data manifolds and intersections}
\label{sec:non-connected}

Earlier we assumed that the data comes from a connected manifold $M$, whose local dimension is constant. Moreover, it was embedded in $\R^D$, precluding self-intersections. These restrictions can be relaxed as follows. Firstly, we may allow $M$ to contain multiple connected components. Secondly, instead of an embedding, we may consider a \emph{good} immersion $j\colon M\to \R^D$, satisfying the following finiteness condition.

\begin{definition}
\label{def:good-immersion}
    We will call an immersion $j\colon M\to N$ \emph{good}, if $M$ admits an open cover $\C$ such that for every $U\in\C$ the restriction of $j$ to $U$ is an embedding, and moreover, every $x\in N$ has an open neighborhood whose preimage intersects only finitely many sets in $\C$.
\end{definition}

In the non-connected case, the dimension is no longer constant on the manifold, but can differ between its components. We will denote by $\dim_x M$ the dimension of $M$ at a point $x\in M$.

Before we proceed, we will prove a simple technical lemma.

\begin{lemma}
\label{lem:finite-component-intersection}
    Let $j\colon M\to N$ be a good immersion. Then every $x\in N$ has a neighborhood whose preimage intersects only finitely many connected components of $M$.
\end{lemma}

\begin{proof}
    Let $\C$ be an open cover of $M$ satisfying conditions of Definition~\ref{def:good-immersion}. Take $x\in N$, and let $V \subset N$ be a neighborhood of $x$ such that $j^{-1}(V)$ intersects only finitely many sets $U_1,\dots, U_n\in\C$. On each $U_i$ the restriction of $j$ is an embedding, so there exists a neighborhood $V_i\subset V$ of $x$ whose preimage is contained in a single connected component of $U_i$, and hence in a single connected component $M_i$ of $M$. 
    
    The intersection $\bigcap_i V_i$ is the required neighborhood of $x$, as its preimage is contained in the finite union of connected components $\bigcup_i M_i$.
\end{proof}

This more general case reduces to the one studied in Section~\ref{sec:core-estimate}, as the following reasoning shows.

\begin{proposition}
\label{prop:resolution-of-intersections}
    Suppose $j\colon M\to N$ is an immersion of manifolds. Moreover, let $\mu$ be a smooth measure on $M$. Then there exists a manifold $\tilde{M}$ endowed with a measure $\tilde{\mu}$ and a local diffeomorphism $f\colon\tilde{M}\to M$, such that 
    \begin{enumerate}
        \item the measure $\tilde{\mu}$ is smooth
        \item the pushforward $f_*\tilde{\mu}$ equals $\mu$;
        \item $\tilde\jmath = j\circ f\colon \tilde{M}\to N$ restricted to every connected component of $\tilde{M}$ is an embedding.
        \item if $j$ is good, then so is $\tilde\jmath$;

    \end{enumerate}    
\end{proposition}

\begin{proof}
    Since $j$ is an immersion, there exists an open cover $\C$ of $M$ such that on every $U\in\C$ the restriction of $j$ is an embedding. Let $\{\psi_U : U\in\C\}$ be a partition of unity subordinate to $\C$. Denote $M_U = \{x\in M : \phi_U(x) > 0\}$, and let $f_U\colon M_U\to M$ be the corresponding inclusion map. Finally, let $\tilde{M}$ be the disjoint union of $\{M_U : U\in\C\}$, and define $f\colon \tilde{M}\to M$ by gluing together the inclusions $f_U$.
    
    The measure $\tilde{\mu}$ can be defined as
    \begin{equation}
        \tilde{\mu} = \sum_{U\in\C} (f_U^{-1})_*(\psi_U \mu),
    \end{equation}
    i.e.\ for every $U\in\C$ we multiply $\mu$ by density function $\psi_U$, restrict it to $M_U$ and pull it to $\tilde{M}$ through $f_U$. Since by definition $\psi_U$ is continuous and positive on $M_U$, the measure $\tilde{\mu}$ is smooth. Moreover, by construction we have $f_*\tilde{\mu} = \mu$,     and the restriction of $\tilde\jmath$ to every $M_U$ (and therefore every to every connected component) is an embedding.

    To show the last assertion, assume that $j$ is good. In this case, the cover $\C$ defined above can be chosen in such a way that for every $x\in N$ there exists a neighborhood $V\subset N$ whose preimage $j^{-1}(V)$ intersects only finitely many sets in $\C$. It is then easy to see that the cover $\{M_U : U\in\C\}$ of $\tilde{M}$ satisfies the conditions of Definition~\ref{def:good-immersion}, so $\tilde\jmath$ is good.
\end{proof}

Now, suppose that in Theorem~\ref{thm:core-estimate}, instead of an embedded submanifold $S$, we are dealing with the image of a proper immersion $j\colon M\to \R^D$, and that $p_S$ is the pushforward of a probability measure $\mu$ on $M$. Thanks to Proposition~\ref{prop:resolution-of-intersections}, this reduces to the situation where $j$ restricted to every connected component of $M$ is an embedding.

\begin{proposition}
    Suppose $j\colon M\to \R^D$ is a good immersion, and its restriction to every connected component of $M$ is an embedding. Let $\mu$ be a smooth probability measure on $M$, and $p_S=j_*\mu$. For $x\in S = j(M)$ and sufficiently small $\delta$ we have
    \begin{equation}
        \log\rho_\delta(x) = (d-D)\log\delta + O(1),
    \end{equation}
    where 
    \begin{equation}
        d = \min_{j(y)=x} \dim_y M.
    \end{equation}
\end{proposition}

\begin{proof}
    By Lemma~\ref{lem:finite-component-intersection}, for sufficiently small $r$ the preimage $j^{-1}(B)$ of the ball $B=B(x,r)$ centered at $x$ intersects only finitely many connected components of $M$. Denote them by $M_1,\dots, M_k$, and let $M_0$ be the union of the remaining components. The measure $\mu$ can be decomposed as
    \begin{equation}
        \mu = \sum_{i=0}^k \mu(M_i)\mu_i,
    \end{equation}
    where $\mu_i$ is the restriction of $\mu$ to $M_i$, normalized to a probability measure. If we put $p_i = j_*\mu_i$, a similar decomposition holds for $p_S$.
    
    If we apply Theorem~\ref{thm:core-estimate} to $j(M_i)$ endowed with the measure $p_i$, for $i>0$, the corresponding perturbed density $\rho^i_\delta$ satisfies
    \begin{equation}
        \rho^i_\delta(x) \asymp \delta^{\dim{M_i}-D}
    \end{equation}
    for sufficiently small $\delta$. Moreover, for $\delta < r^2$, we have $j(M_0) = j(M_0)\setminus B(x,\delta^{1/2})$, so by Lemma~\ref{lem:gaussian-estimate-outside-ball}
    \begin{equation}
        \lim_{\delta\to 0^+} \rho^0_\delta(x) = 0.
    \end{equation}
    Consequently, for small $\delta<1$
    \begin{equation}
        \rho_\delta(x) = \sum_{i=0}^k \mu(M_i)\rho^i_\delta(x) \asymp \sum_{i=1}^k \delta^{\dim{M_i}-D},
    \end{equation}
    and the term with the lowest exponent dominates.
\end{proof}

\section{Examples with explicit derivations}

\subsection{A motivating example}
\label{sec:motivating-example}

Consider the standard embedding $\R^d \subset \R^D$. Take for $S$ a bounded open subset of $\R^d$, endowed with the uniform probability measure $p_S$ with constant density $\rho\equiv\operatorname{vol}(S)^{-1}$ on $S$. If we denote by $x_1$ and $x_2$ the components of a vector $x\in \R^D$ corresponding to the standard decomposition $\R^D = \R^d \times \R^{D-d}$, it follows from \eqref{eq:mollified-density} and properties of the Gaussian function, that
\begin{equation}
    \label{eq:flat-example-mollified-density}
    \rho_\delta(x) = 
    \frac{\phi^{D-d}_{\delta}(x_2)}{\operatorname{vol}(S)} 
    \int_S \phi^d_{\delta}(x_1-y_1) \, d y_1.
\end{equation}

Now, if $x$ is an interior point of $S$, then $x_2=0$. Moreover, for sufficiently small $\delta$, the integral above is arbitrarily close to $1$, as most of the mass of the integrand falls into a small neighborhood of $x_1$, which is contained in $S$. Therefore, for sufficiently small $\delta$
\begin{equation}
    \rho_\delta(x) \asymp \phi^{D-d}_\delta(0) = \delta^{d-D}\phi^{D-d}(0) \asymp \delta^{d-D}
\end{equation}
uniformly in $\delta$.

It follows that
\begin{equation}
    \log \rho_\delta(x) = (d-D) \log\delta + O(1),
\end{equation}
and hence
\begin{equation}
    d-D = \lim_{\delta\to 0} \frac{\log \rho_\delta(x)}{\log\delta}.
\end{equation}
In practice, $d-D$, and in consequence $d$, can be estimated by considering $\rho_\delta(x)$ for multiple small values of $\delta$, and using linear regression.

\subsection{Normal distribution in $\R^D$}
\label{sec:delta-gaussian-dimensions}
Suppose that $S=\R^D$, and $p_S = \N(0, \Sigma)$, where $\Sigma$ is a diagonal matrix with entries $\sigma_1^2 \geq \sigma_2^2 \geq\dots\geq \sigma_D^2$. In this case, the perturbation with $\N(0, \delta^2 I)$ yields another normal distribution $\N(0, \Sigma + \delta^2 I)$, whose density at $0$ is
\begin{equation}
\label{eq:perturbed-normal-distribution-at-0}
    \rho_\delta(0) = (2\pi)^{-D/2}\prod_{k=1}^D (\sigma_k^2 + \delta^2)^{-1/2}.
\end{equation}

\begin{proposition}
    \label{prop:delta-gaussian-dimensions}
    Let $1 \leq d < D$, and denote $\tau = (\sigma_d\sigma_{d+1})^{1/2}$. For $\lambda \geq 1$ and $\delta\in [\lambda^{-1}\tau, \lambda\tau]$ we have
    \begin{equation}
        \log\rho_\delta(0) = (d-D)\log\delta + M - C_\lambda,
    \end{equation}
    where $M$ is independent of $\delta$, and  $0 \leq C_\lambda \leq \frac{D\sigma_{d+1}}{2\sigma_d} \lambda^2$.
\end{proposition}
In other words, the above proposition states that for $\delta$ between two consecutive deviations $\sigma_d$ and $\sigma_{d+1}$, our LID estimate is approximately the number $d$ of dimensions in which the Gaussian distribution is `thicker' than $\delta$, and the approximation error decreases with the growth of the ratio $\sigma_d/\sigma_{d+1}$ and the distance of $\delta$ from $\sigma_d$ and $\sigma_{d+1}$.

\begin{proof}
    Let us denote $\eta = \lambda(\sigma_{d+1}/\sigma_d)^{1/2}$. The, for $k\leq d$ we may compute $\lambda\tau = \eta\sigma_d \leq \eta\sigma_k$, which leads to 
    \begin{equation}
        \sigma_k^2 + \delta^2
        \leq  (1 + \eta^2)\sigma_k^2.
    \end{equation}
    On the other hand, for $k \geq d+1$, we have $\lambda^{-1}\tau = \eta^{-1}\sigma_{d+1} \geq \eta^{-1}\sigma_k$, and similarly to the previous case, we have
    \begin{equation}
        \sigma_k^2 + \delta^2
        \leq 
        (1+\eta^2)\delta^2.
    \end{equation}
    By applying these two estimates to the formula~\eqref{eq:perturbed-normal-distribution-at-0} for $\rho_\delta(0)$ we are able to obtain a two-sided estimate
    \begin{equation}
        M(1+\eta^2)^{-D/2} \delta^{d-D}
        \leq \rho_\delta(0) 
        \leq M\delta^{d-D}, 
    \end{equation}
    with $M=(2\pi)^{-D/2} \prod_{k=1}^d \sigma_k^{-1}$ independent of $\delta$. Finally, after taking $\log$ we can see that
    \begin{equation}
        \log \rho_\delta(0) = (d-D)\log\delta + \log M - \frac{D}{2}\log(1+\eta^2),
    \end{equation}
    and the last term is positive and bounded from above by $D\eta^2/2$, yielding the desired estimate by substituting $\eta$.
\end{proof}

From the above Proposition we can see that if there is a large gap between $\sigma_d$ and $\sigma_{d+1}$, then for $\delta$ in the neighborhood of their geometric mean, the LID estimate obtained through linear regression should be approximately $d$, with approximation error decreasing, and the range of viable $\delta$ increasing with the growth of the gap size, expressed by the ratio $\sigma_{d+1}/\sigma_d$.

\subsection{Points along a line}

Consider a zero-dimensional manifold $M$, consisting of $N$ points, endowed with uniform probability measure. Suppose $M$ is embedded into $\R^D$ in such a way that its image $\set{x_1, \dots, x_N}$ is actually contained in $\R \subset \R^D$, and has the form $x_k = (\xi_k, 0, \dots, 0)$, where $\xi_{k+1} \geq \xi_k + \eta$ for some $\eta > 0$, i.e.\ the indexing is chosen in such a way that the points $x_k$ are ordered along $\R$, and the distances between them are at least $\eta$.

In this setting, we will study the quantity $\rho_\delta(x_n)$ more closely, and attempt to understand its relationship with the perturbation magnitude for any $\delta$, not just sufficiently small ones. We have
\begin{equation}
    \rho_\delta(x_0) 
    = \frac{1}{N}\sum_{k=1}^N \phi^D_\delta(x_n-x_k) 
    = \frac{\phi^D_\delta(0)}{N} \Bigg( 1 + \sum_{\substack{k =1\\ k\ne n}}^N \frac{\phi^D_\delta(x_n-x_k) }{\phi^D_\delta(0)} \Bigg)
    = M \delta^{-D} \left( 1 + \epsilon_\delta\right),
\end{equation}
where $M=({N(2\pi)^{D/2}})^{-1}$, and
\begin{equation}
    \epsilon_\delta 
    = \sum_{\substack{k =1\\ k\ne n}}^N \frac{\phi^D_\delta(x_n-x_k)}{\phi^D_\delta(0)}
    = \sum_{\substack{k =1\\ k\ne n}}^N \exp \left[ -\frac{1}{2} \left(\frac{\xi_n-\xi_k}{\delta}\right)^2\right].
\end{equation}

After taking $\log$, we get
\begin{equation}
    \log \rho_\delta(x_0) = -D\log\delta + \log M + \log(1+ \epsilon_\delta),
\end{equation}
where the term $\log M$ is independent of $\delta$, and $0 \leq \log(1+\epsilon_\delta) \leq \epsilon_\delta$.

\begin{proposition}
\label{prop:points-along-line-upper-bound}
    Let $\lambda \geq 1$. If $\delta < \eta / (\sqrt{2}\lambda)$ then $\epsilon_\delta \leq 4e^{-\lambda^2}$.  In particular, for $\epsilon > 0$, we have $\epsilon_\delta < \epsilon$ provided that 
    \begin{equation}
        \delta < \frac{\eta}{(-2\log(\epsilon/4))^{1/2}},
    \end{equation}
    i.e.\ the threshold value for $\delta$ depends logarithmically on $\epsilon$.
\end{proposition}
\begin{proof}
    We have $\abs{\xi_i - \xi_j} \geq \eta\abs{i-j}$, and therefore
    \begin{equation}
        \epsilon_\delta \leq \sum_{\substack{k =1\\ k\ne n}}^N \exp \left[ -\frac{1}{2} \left(\frac{\eta(n-k)}{\delta}\right)^2\right]
        \leq \sum_{\substack{k =1\\ k\ne n}}^N e^ {-{\lambda^2} (n-k)^2}.
    \end{equation}
    For an upper estimate, we may also extend the summation over all integers except $n$, obtaining 
    \begin{equation}
        \epsilon_\delta 
        \leq \sum_{k\ne n} e^ {-{\lambda^2} (n-k)^2}
        = 2\sum_{j=1}^\infty e^ {-{\lambda^2} j^2}
        \leq 2\sum_{j=1}^\infty e^ {-{\lambda^2} j}
        = \frac{2}{1-e^{-\lambda^2}} e^{-\lambda^2}.
    \end{equation}
    For $\lambda \geq 1$ we have $(1-e^{-\lambda^2})^{-1} \leq 2$, so in the end $\epsilon_\delta \leq 4e^{-\lambda^2}$. By solving $\epsilon = 4e^{-\lambda^2}$ for $\lambda$ we obtain $\lambda = (-\log(\epsilon/4))^{1/2}$, yielding the last assertion.
\end{proof}

\section{Ideal LIDL for normal distribution on a line}
\label{sec:appendix-lidl-for-normal-distribution}

Suppose our submanifold $S$ is the image of the standard embedding $\R\subset\R^D$, and let $p_S=\N(0,1)$. In this case, the perturbed distribution is $\N(0,\Sigma)$, where $\Sigma$ is a diagonal matrix with entries $(1+\delta^2, \delta^2, \dots, \delta^2)$. The density $\rho_\delta$ at a point $x=(t,0,\dots,0)\in S$ is therefore
\begin{equation}
    \rho_\delta(x) =  \frac{\delta^{1-D}}{(2\pi)^{D/2}(1+\delta^2)^{1/2}} \exp\left( -\frac{t^2}{2(1+\delta^2)}\right),
\end{equation}
and its logarithm can be decomposed into the following sum
\begin{equation}
    \log\rho_\delta(x) = 
        (1-D)\log\delta 
        - \frac{D}{2}\log(2\pi)
        - \frac{1}{2}\log(1+\delta^2)
        - \frac{t^2}{2(1+\delta^2)}.
\end{equation}
Let us now apply the trivial case of linear regression involving only two points, amounting to computing the slope of the line passing through two points. We have
\begin{equation}
    \frac{\log\rho_{\delta_1}(x) - \log\rho_{\delta_2}(x)}{\log\delta_1 - \log\delta_2} = 
        1 - D - \epsilon(t)
        ,
\end{equation}
where the error term expands to
\begin{equation}
    \epsilon(t)= \frac{1}{2(\log\delta_1 - \log\delta_2)}
            \left[ \log\frac{1+\delta_1^2}{1+\delta_2^2}
            - t^2 \left( \frac{1}{1+\delta_1^2} - \frac{1}{1+\delta_2^2} \right)\right],
\end{equation}
yielding a LID estimate $\hat{d}_x = 1 - \epsilon(t)$ at $x$.
We can see that the error $\epsilon$ decomposes into two terms of opposite signs. The first term depends only on $\delta$, and the second one, grows quadratically with $t$.

If we put $\delta_1=\eta\delta$, and $\delta_2=\delta$, the coefficient of $t^2$ can be further rewritten as
\begin{equation}
    \frac{1}{2(\log\delta_1 - \log\delta_2)} \left( \frac{1}{1+\delta_1^2} - \frac{1}{1+\delta_2^2} \right) 
    = \frac{\delta^2(1-\eta^2)}{2\log\eta(1+\delta^2)(1+(\delta\eta)^2)} \asymp \frac{\delta^2(1-\eta^2)}{2\log\eta},
\end{equation}
where the estimate holds uniformly in $\delta$ if $\delta$ is bounded $\delta$ from above. Although for fixed $\delta$ and $\eta$ the error is unbounded as a function of $t$, if we were allowed to adjust $\delta$ based on $t$ (with fixed $\eta$), for the error $\epsilon(t)$ to be bounded in $t$ it is necessary and sufficient that $\delta \leq C/t$ for some constant $C$.

Finally, the expected error for the LID estimate (computed in the above manner) at a random $x$ drawn from our distribution can be computed
\begin{equation}
\begin{split}
    \int_\R \epsilon(t)\phi^1(t) \,dt 
    & = \frac{1}{2(\log\delta_1 - \log\delta_2)}
            \left[ \log\frac{1+\delta_1^2}{1+\delta_2^2}
            - \int_\R t^2 \phi^1(t) \,dt\left( \frac{1}{1+\delta_1^2} - \frac{1}{1+\delta_2^2} \right)\right] = \\
    & = \frac{1}{2(\log\delta_1 - \log\delta_2)}
            \left[ \log\frac{1+\delta_1^2}{1+\delta_2^2}
            - \left( \frac{1}{1+\delta_1^2} - \frac{1}{1+\delta_2^2} \right)\right] 
    = \epsilon(1),
\end{split}
\end{equation}
where the last integral is just the variance of $\N(0,1)$, i.e.\ $1$.

\section{Normalizing Flows}
\label{sec:nf}

NF are very flexible tools for approximating probability distributions. They use parametrized nonlinear invertible transformation $f_\theta$ and change of variable formula to transform a simple density $\pi(z)$ into a more complicated one. NF are trained using gradient-based methods (e.g. SGD) to maximize log-likelihood of the data
$$\max_\theta \sum_{i}\log q(x_i)$$
where
$$q(x) = \pi(f_\theta(x))\left|\det\frac{\partial f_\theta(x)}{\partial x}\right|.$$
We used MAF \citep{papamakarios2017masked}, RQ-NSF \citep{durkan2019neural} and Glow \citep{kingma2018glow} models in our experiments. More detailed introduction to normalizing flows can be found in \cite{dinh2014nice}.

\include{table}

\section{Experimental details}
\label{sec:exp-details}

In this section we present some results of additional experiments, some details and other observations. 

When using LIDL with parametric density estimators on non-synthetic datasets, choosing hyperparameters is a challenge. We cannot directly estimate the error of the algorithm because we does not have access to ground truth LID. However, we observed in our experiments that choosing the hyperparameters leading to models minimizing negative log-likelihood on the validation set is a good strategy for minimizing the error of the LID estimate. We apply this approach in all our experiments; as density estimators we employ MAF \citep{papamakarios2017masked}, RQ-NSF \citep{durkan2019neural} and Glow \citep{kingma2018glow}.

In scalability experiments we used 3 types of datasets. Uniform distribution on interval $(0,1)$ on a hypercube (denoted by $\U_N$, where $N$ is dimensionality of a cube), multivariate Gaussian ($\N_N \subseteq \R^N$) where $N$ is dimensionality of a distribution and data space, and ($\N_N \subseteq \R^{2N}$), where we embedded $N$-dimensional Gaussian in $2N$-dimensional space by duplicating each coordinate.
In each experiment we used 11 $\delta$s  between 0.025 and 0.1.

\subsection{Reducing the error of the density estimate}
\label{sec:exp-reducing}
Because model ensemble methods \citep{opitz1999popular} often reduces prediction error in many machine learning models, and most of LIDL error comes from the imperfect density estimators, we applied it to our problem by increasing the number of models $n$ used in LIDL. We were able to reduce an error of each estimate by simply adding more models between the same range of $\delta$s. An example of this behavior for 10-dimensional Gaussian embedded in 20-dimensional space is plotted in Fig.\ref{fig:mse-n-models} in Appendix~\ref{sec:exp-details}.

\begin{figure}[tb]
    \centering
    \includegraphics[width=0.4\textwidth]{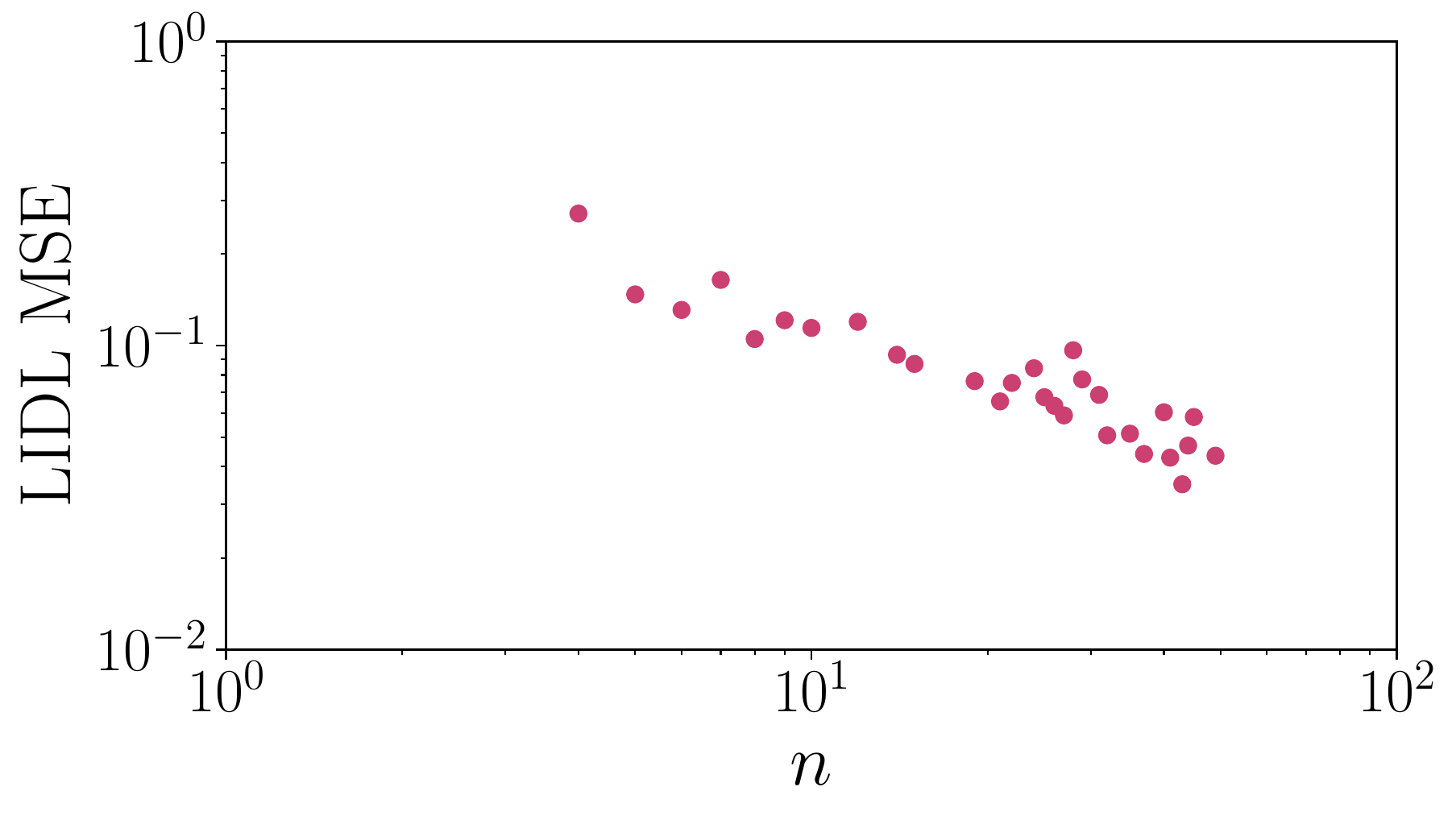}
    \caption{The dependence of mean-square error (MSE) on the number of models used in LIDL ($n$ from Algorithm\ref{alg:lidl}). We can observe monothonic decrease of the estimate error with the increase of $n$.}
    \label{fig:mse-n-models}
\end{figure}

\begin{figure}[h]
    \centering
    \includegraphics[width=0.4\textwidth]{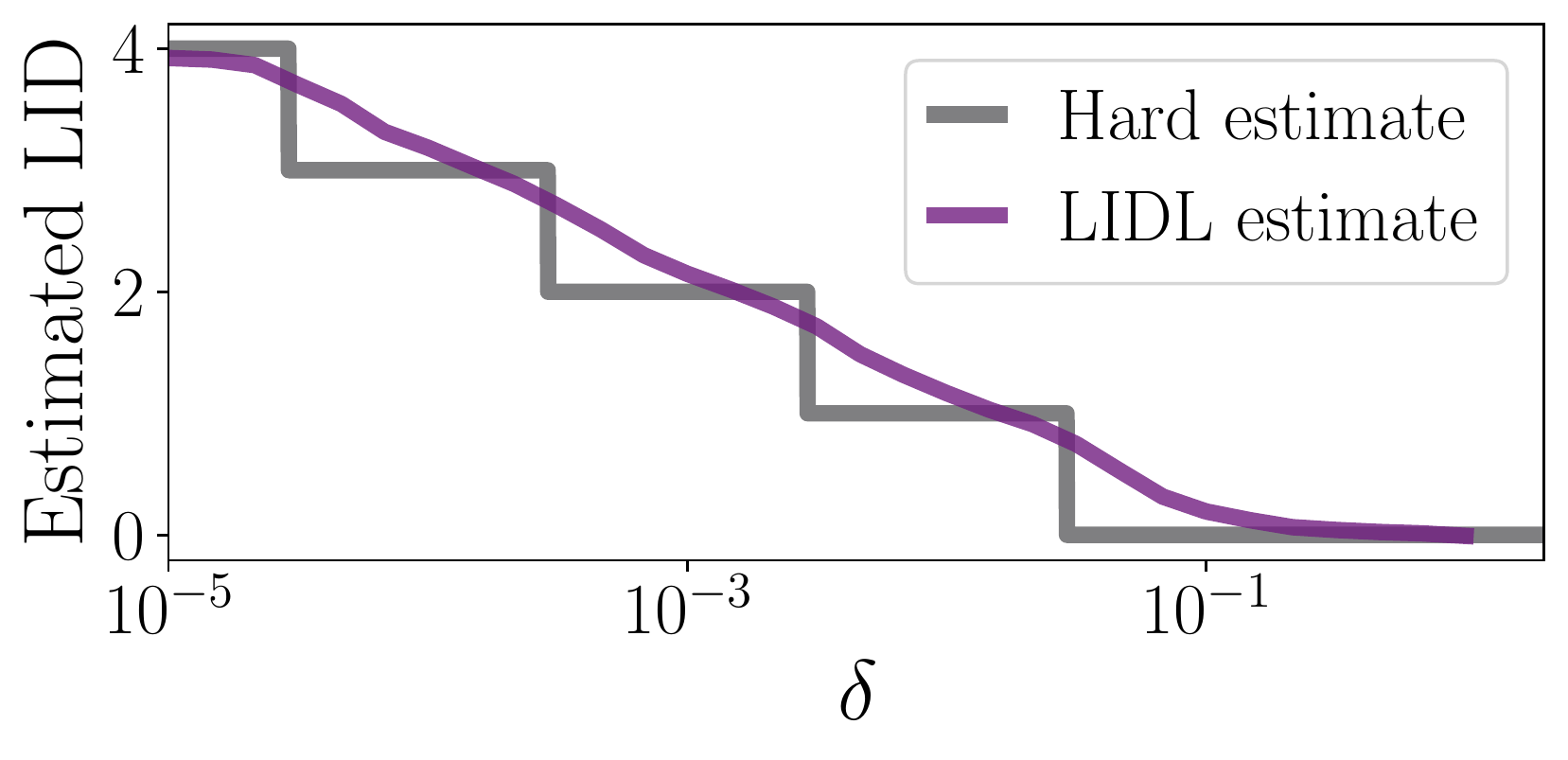}
    \caption{LIDL and hard estimate for different values of $\delta$ for 4-dimensional multiscale uniform distribution. We can see that LIDL ignores dimensions that are much smaller than $\delta$ even with imperfect density estimators.}
    \label{fig:step-delta-uniform}
\end{figure}

\begin{figure}[h]
\centering
\includegraphics[width=0.9\textwidth]{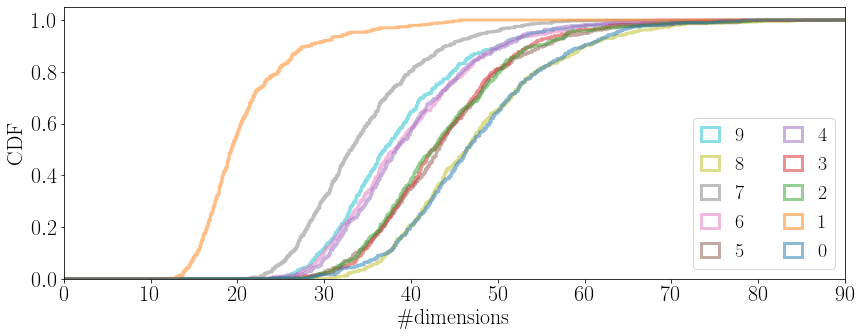}
\caption{Empirical CDF of 5000 examples from MNIST dataset. Each line represents CDF for separate class in the dataset. Class number (which also is a represented digit in this case) can be found in the legend.}
\label{fig:mnist-cdf}
\end{figure}
\begin{figure}[h]
\centering
\includegraphics[width=0.9\textwidth]{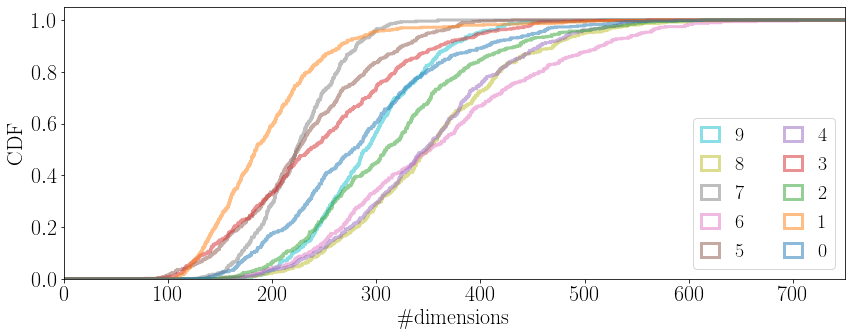}
\caption{Empirical CDF of 5000 examples from FMNIST dataset. Each line represents CDF for separate class in the dataset. Class number can be found in the legend.}
\label{fig:fmnist-cdf}
\end{figure}

\subsection{Image datasets}
\label{sec:details-image}
We present cumulative distribution function (CDF) for MNIST and FMNIST in Fig.~\ref{fig:mnist-cdf}~and~\ref{fig:fmnist-cdf}.  More samples from MNIST, FMNIST, Celeb-A sorted by their LID can be found in Fig.~\ref{fig:mnist-sort},~\ref{fig:fmnist-sort},~and~\ref{fig:celeba-sort}.

We can observe, that dimensionality estimates obtained from LIDL on MNIST are higher than those reported in \cite{pope2020intrinsic} or \cite{kleindessner2015dimensionality} and obviously depend on choosing the range of $\delta$s. We want to present some argument for our estimates:
\cite{cavallari2018unsupervised} used auto-encoder representation of MNIST as an input to SVN digit classifier and they achieved the best classification results for an auto-encoder with latent space size greater than 100. This means that we need more than 100 dimensions to encode an average MNIST digit to preserve all the information about it. Of course the compression with autoencoder is not ideal, but we can argue that true ID lies somewhere between.

\paragraph{Relation between examples for different range of $\delta$s}
We observed that for image dataset LID estimates for two disjoint sets of 4 $\delta$s have similar ranks (they on average differ between 10\%-15\%), and relations between points in each set (i.e. if LID estimate for $x_j$ is lower than LID estimate for $x_i$) are preserved in 80-90\% of cases. This of course depends strongly on $\delta$ range and the dataset.

\paragraph{LID estimate dependence on $\delta$ and effect of dequantization}

We present MNIST and FMNIST LID estimates (averaged per class) dependence on $\delta$ in Fig.~\ref{fig:mnist-delta-dependence}~and~\ref{fig:fmnist-delta-dependence}. Images present wide range of $\delta$s (from $10^{-4}$ to $10^{1}$) for original datasets and datasets with dequantization used after training. Black dashed line indicates a theoretical $\delta$, above which LIDL should not calculate dequantization dimensions into LIDL estimate. This is 10 times standard deviation of dequantization noise $\mathcal{U}(0,1/255)$. We can see that slightly above this threshold estimates for quantized and dequantized datasets aligh with each other. We can also observe, that for dequantized datasets and very small $\delta$ LID estimate is close to the dimensionality of the space, and for very big $\delta$s, LID estimates are close to 0 as expected.

\input{lid-ml-performance}

\begin{figure}[h]
\centering
\includegraphics[width=0.99\textwidth]{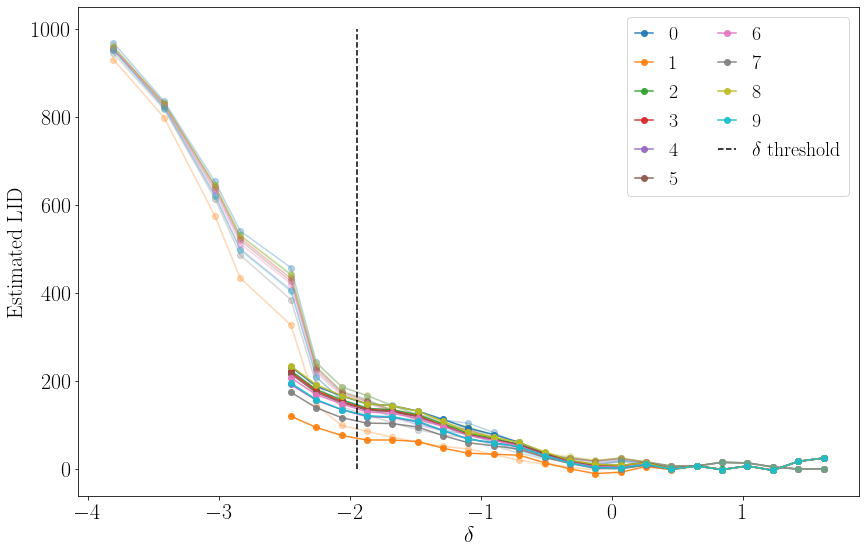}
\caption{MNIST average LID estimates for each class for quantized (strong color) and dequantized (faded colors) as a function of $\delta$.}
\label{fig:mnist-delta-dependence}
\end{figure}
\begin{figure}[h]
\centering
\includegraphics[width=0.99\textwidth]{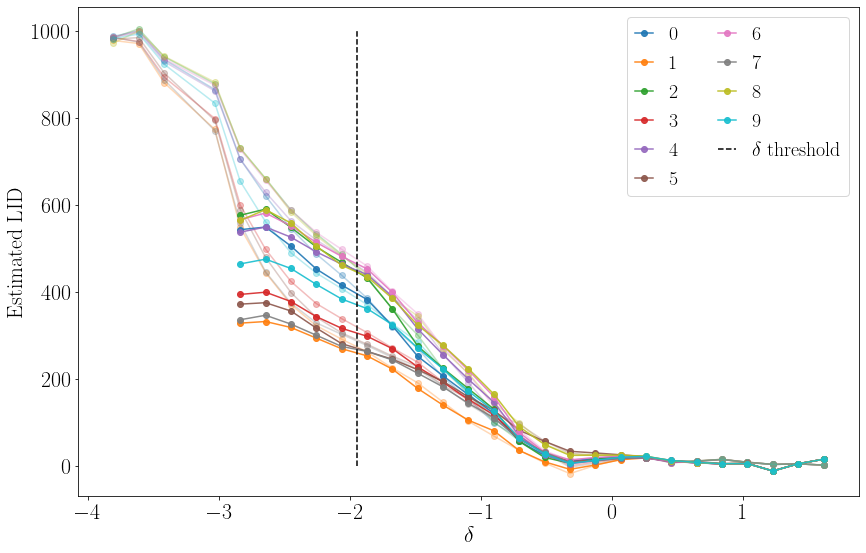}
\caption{FMNIST average LID estimates for each class for quantized (strong color) and dequantized (faded colors) as a function of $\delta$.}
\label{fig:fmnist-delta-dependence}
\end{figure}

\begin{figure}[t]
\centering
\includegraphics[width=0.95\textwidth]{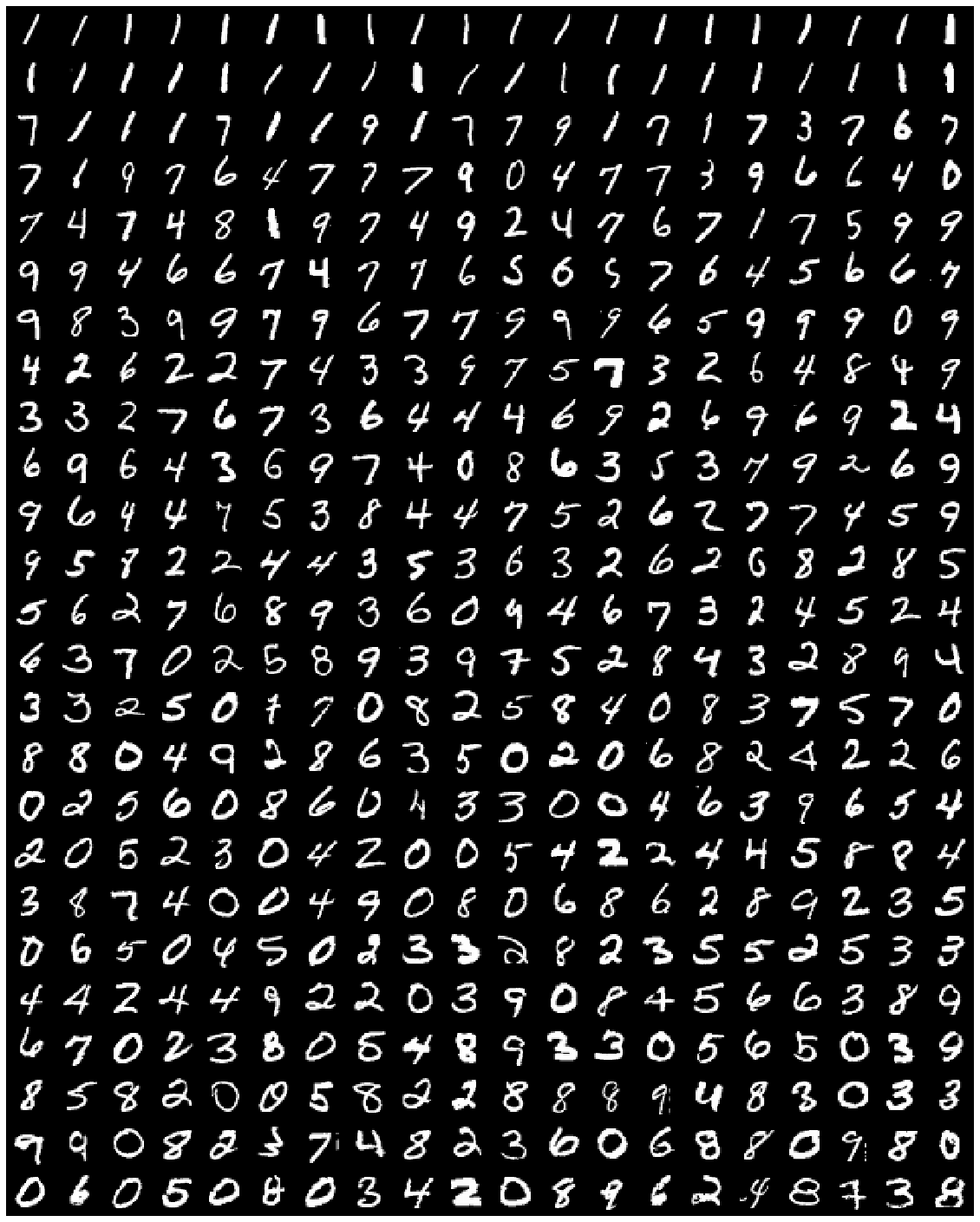}
\caption{Samples from MNIST sorted by their LID estimates.}
\label{fig:mnist-sort}
\end{figure}
\begin{figure}[t]
\centering
\includegraphics[width=0.95\textwidth]{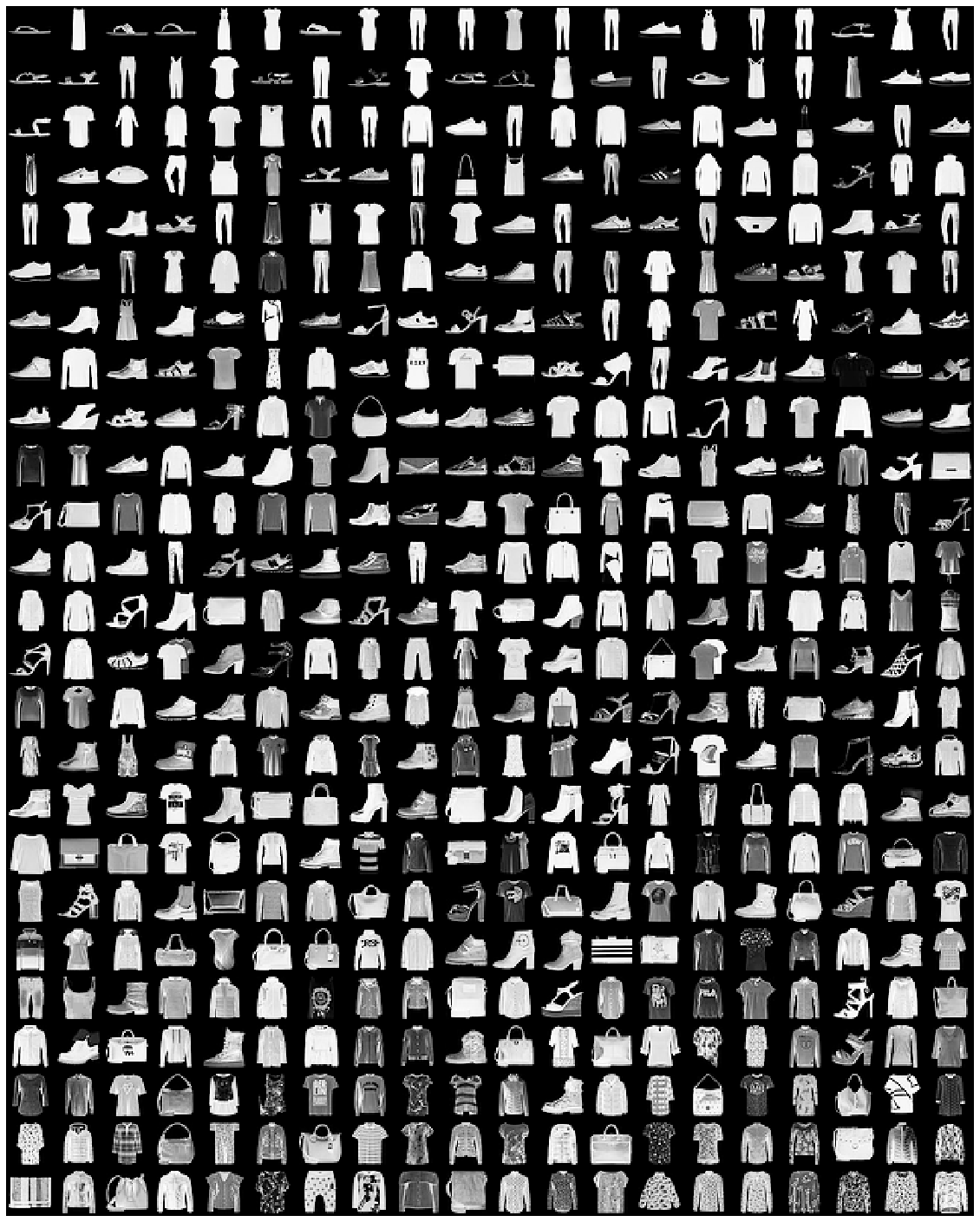}
\caption{Samples from FMNIST sorted by their LID estimates.}
\label{fig:fmnist-sort}
\end{figure}
\begin{figure}[t]
\centering
\includegraphics[width=0.95\textwidth]{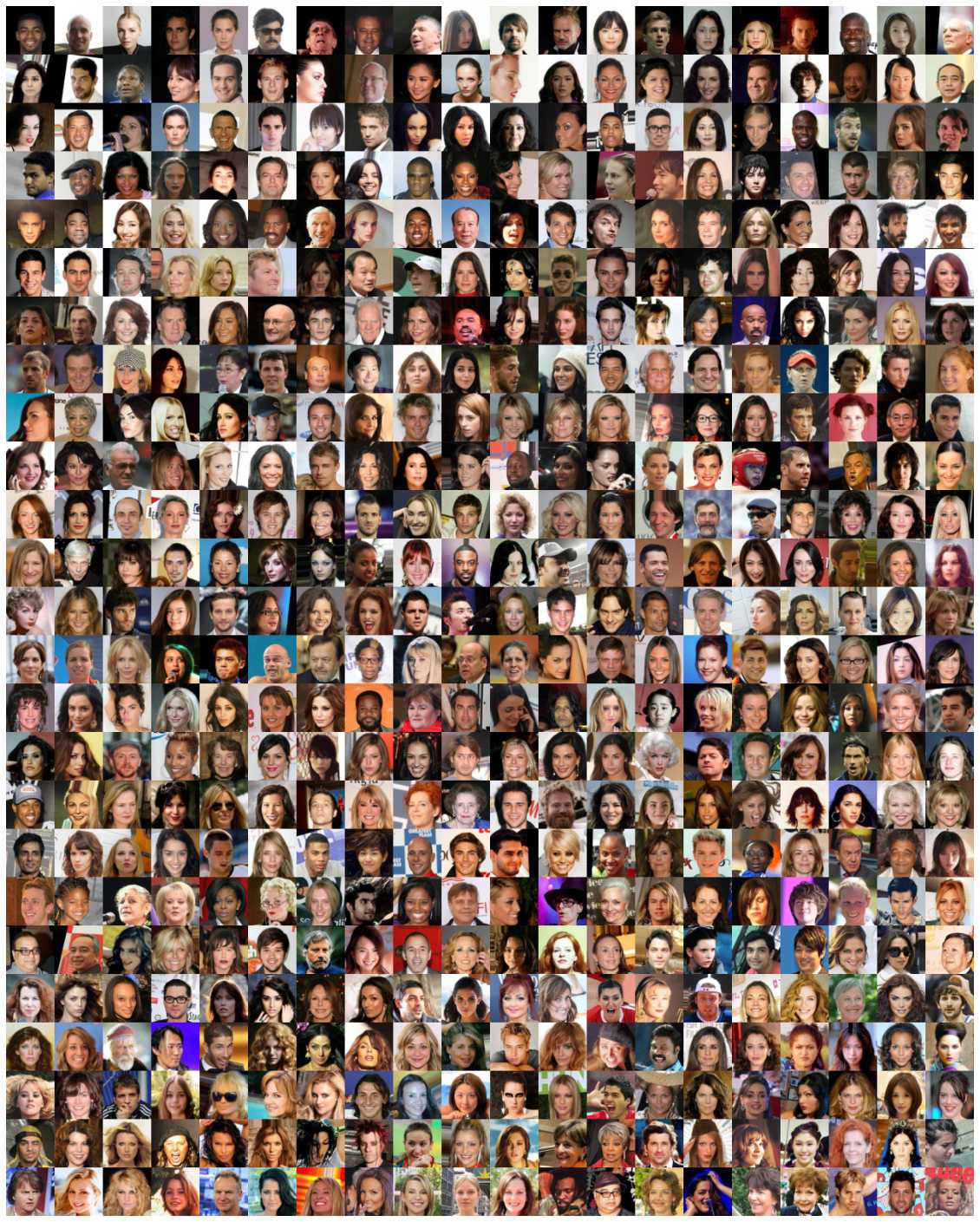}
\caption{Samples from Celeb-A sorted by their LID estimates.}
\label{fig:celeba-sort}
\end{figure}

%% file: table.tex
\begin{table}[h!]
\centering
\captionof{table}{Relative bias of LID estimates. All algorithm names explained in Table~\ref{tab:skdim}}\label{tab:comparison}
\scalebox{0.88}{
\begin{tabular}{ | c | c || c | c || c | c | c | c | c |  c | c | c | c | c |}
 \hline
 Distribution & LID &LIDL$_{M}$&LIDL$_{R}$&COR&ESS&KNN&LPC& MAD&MIN&MLE&MOM&TLE&TWO \\ \hline\hline

Lollipop in $\R^2$&0&\textbf{0.00}&\textbf{0.00}&1.67&1.67&1.60&1.67&1.82&1.67&1.80&1.74&-&1.65\\\hline
Lollipop in $\R^2$&1&\textbf{0.00}&\textbf{0.01}&\textbf{0.00}&\textbf{0.00}&0.58&\textbf{0.01}&0.10&\textbf{0.00}&\textbf{0.05}&\textbf{0.00}&-&\textbf{0.00}\\\hline
Lollipop in $\R^2$&2&\textbf{-0.00}&\textbf{-0.00}&\textbf{-0.01}&\textbf{-0.00}&-0.07&\textbf{0.00}&0.10&\textbf{-0.00}&0.08&\textbf{0.01}&-&\textbf{-0.03}\\\hline
$\U$ on helix in $\R^3$&1&\textbf{0.01}&\textbf{0.00}&\textbf{0.00}&\textbf{0.00}&0.68&\textbf{0.00}&0.12&\textbf{0.00}&0.06&\textbf{0.00}&\textbf{0.00}&\textbf{0.00}\\\hline
$\U$ on $S^7 \subseteq \R^8$&7&\textbf{-0.00}&\textbf{0.00}&-0.28&\textbf{0.00}&-0.37&\textbf{0.00}&\textbf{0.03}&-0.18&\textbf{0.02}&\textbf{-0.04}&0.08&-0.13\\\hline
Swiss roll in $\R^3$&2&\textbf{0.04}&\textbf{0.01}&\textbf{-0.00}&\textbf{0.00}&\textbf{-0.05}&\textbf{0.00}&0.06&\textbf{0.00}&0.05&\textbf{0.00}&0.05&\textbf{-0.01}\\\hline
$\N_{10} \subseteq \R^{10}$&10&\textbf{0.00}&\textbf{0.00}&-0.40&\textbf{-0.00}&-0.47&\textbf{0.00}&\textbf{0.02}&-0.25&\textbf{0.01}&-0.07&\textbf{0.01}&-0.16\\\hline
$\N_{100} \subseteq \R^{100}$&100&\textbf{-0.00}&\textbf{0.00}&-0.78&\textbf{-0.01}&-0.66&-0.28&-0.51&-0.90&-0.50&-0.57&-0.56&-0.60\\\hline
$\N_{1000} \subseteq \R^{1000}$&1000&\textbf{0.00}&\textbf{0.00}&-0.93&-0.09&-0.74&-0.90&-0.83&-0.99&-0.82&-0.85&-0.85&-0.86\\\hline
$\N_{4000} \subseteq \R^{4000}$&4000&\textbf{-0.00}&-&-0.96&-0.29&-0.77&-0.98&-0.91&-1.00&-0.91&-0.92&-0.93&-0.93\\\hline
$\N_{10}\subseteq \R^{20}$&10&\textbf{0.00}&\textbf{0.01}&-0.40&\textbf{-0.00}&-0.25&\textbf{0.00}&\textbf{0.02}&-0.25&\textbf{0.01}&-0.07&\textbf{0.01}&-0.16\\\hline
$\N_{100} \subseteq \R^{200}$&100&\textbf{0.04}&\textbf{0.03}&-0.78&\textbf{-0.01}&-0.46&-0.28&-0.51&-0.90&-0.50&-0.57&-0.56&-0.60\\\hline
$\N_{1000}\subseteq \R^{2000}$&1000&0.11&0.30&-0.93&-0.09&-0.52&-0.90&-0.83&-0.99&-0.82&-0.85&-0.85&-0.86\\\hline
$\N_{2000}\subseteq \R^{4000}$&2000&0.11&-&-0.95&-0.17&-0.55&-0.95&-0.88&-0.99&-0.87&-0.89&-0.90&-0.90\\\hline
$\U_{10} \subseteq \R^{10}$&10&\textbf{-0.04}&\textbf{-0.04}&-0.39&-0.07&-0.47&\textbf{0.00}&-0.10&-0.29&-0.11&-0.17&-0.07&-0.22\\\hline
$\U_{100} \subseteq \R^{100}$&100&\textbf{0.00}&\textbf{0.01}&-0.75&\textbf{-0.02}&-0.67&-0.28&-0.50&-0.90&-0.49&-0.56&-0.53&-0.59\\\hline
$\U_{1000} \subseteq \R^{1000}$&1000&\textbf{-0.00}&\textbf{0.00}&-0.92&-0.09&-0.76&-0.90&-0.81&-0.99&-0.81&-0.84&-0.84&-0.85\\\hline
$\U_{4000} \subseteq \R^{4000}$&4000&\textbf{-0.00}&-&-0.96&-0.29&-0.78&-0.98&-0.90&-1.00&-0.90&-0.92&-0.92&-0.92\\\hline

 \hline\end{tabular}}
\end{table}

\begin{table}[h!]
\centering
\captionof{table}{Relative MAE of LID estimates. All algorithm names explained in Table~\ref{tab:skdim}}\label{tab:comparison-mae}
\scalebox{0.88}{
\begin{tabular}{ | c | c || c | c || c | c | c | c | c |  c | c | c | c | c |}
 \hline
 Distribution & LID &LIDL$_{M}$&LIDL$_{R}$&COR&ESS&KNN&LPC& MAD&MIN&MLE&MOM&TLE&TWO \\ \hline\hline

Lollipop in $\R^2$&0&\textbf{0.00}&\textbf{0.00}&1.67&1.67&1.60&1.67&1.82&1.67&1.80&1.74&-&1.65\\\hline
Lollipop in $\R^2$&1&\textbf{0.00}&\textbf{0.01}&\textbf{0.00}&\textbf{0.00}&0.58&\textbf{0.01}&0.12&\textbf{0.00}&\textbf{0.05}&\textbf{0.00}&-&\textbf{0.00}\\\hline
Lollipop in $\R^2$&2&\textbf{0.01}&\textbf{0.01}&\textbf{0.01}&\textbf{0.00}&0.62&\textbf{0.00}&0.43&\textbf{0.00}&0.25&\textbf{0.03}&-&\textbf{0.05}\\\hline
$\U$ on helix in $\R^3$&1&\textbf{0.01}&\textbf{0.00}&\textbf{0.00}&\textbf{0.00}&0.68&\textbf{0.00}&0.14&\textbf{0.00}&0.06&\textbf{0.00}&\textbf{0.00}&\textbf{0.00}\\\hline
$\U$ on $S^7 \subseteq \R^8$&7&\textbf{0.00}&\textbf{0.00}&0.28&\textbf{0.00}&0.44&\textbf{0.00}&0.27&0.18&0.18&0.08&0.17&0.15\\\hline
Swiss roll in $\R^3$&2&0.06&\textbf{0.01}&\textbf{0.00}&\textbf{0.00}&0.37&\textbf{0.00}&0.25&\textbf{0.00}&0.14&\textbf{0.02}&0.06&\textbf{0.03}\\\hline
$\N_{10} \subseteq \R^{10}$&10&\textbf{0.00}&\textbf{0.01}&0.40&\textbf{0.00}&0.47&\textbf{0.00}&0.27&0.25&0.19&0.11&0.16&0.17\\\hline
$\N_{100} \subseteq \R^{100}$&100&\textbf{0.00}&\textbf{0.02}&0.78&\textbf{0.02}&0.66&0.28&0.52&0.90&0.50&0.57&0.56&0.60\\\hline
$\N_{1000} \subseteq \R^{1000}$&1000&\textbf{0.00}&\textbf{0.01}&0.93&0.09&0.74&0.90&0.83&0.99&0.82&0.85&0.85&0.86\\\hline
$\N_{4000} \subseteq \R^{4000}$&4000&\textbf{0.01}&-&0.96&0.29&0.77&0.98&0.91&1.00&0.91&0.92&0.93&0.93\\\hline
$\N_{10}\subseteq \R^{20}$&10&\textbf{0.00}&\textbf{0.01}&0.40&\textbf{0.00}&0.68&\textbf{0.00}&0.27&0.25&0.19&0.11&0.16&0.17\\\hline
$\N_{100} \subseteq \R^{200}$&100&\textbf{0.04}&\textbf{0.03}&0.78&\textbf{0.02}&0.87&0.28&0.52&0.90&0.50&0.57&0.56&0.60\\\hline
$\N_{1000}\subseteq \R^{2000}$&1000&0.12&0.30&0.93&0.09&0.95&0.90&0.83&0.99&0.82&0.85&0.85&0.86\\\hline
$\N_{2000}\subseteq \R^{4000}$&2000&0.12&-&0.95&0.17&0.96&0.95&0.88&0.99&0.87&0.89&0.90&0.90\\\hline
$\U_{10} \subseteq \R^{10}$&10&\textbf{0.04}&\textbf{0.04}&0.39&0.07&0.47&\textbf{0.00}&0.27&0.29&0.20&0.18&0.17&0.22\\\hline
$\U_{100} \subseteq \R^{100}$&100&\textbf{0.00}&\textbf{0.02}&0.75&\textbf{0.02}&0.67&0.28&0.50&0.90&0.49&0.56&0.53&0.59\\\hline
$\U_{1000} \subseteq \R^{1000}$&1000&\textbf{0.00}&\textbf{0.02}&0.92&0.09&0.76&0.90&0.81&0.99&0.81&0.84&0.84&0.85\\\hline
$\U_{4000} \subseteq \R^{4000}$&4000&\textbf{0.01}&-&0.96&0.29&0.78&0.98&0.90&1.00&0.90&0.92&0.92&0.92\\\hline

 \hline\end{tabular}}
\end{table}

\begin{table}[h!]
\centering
\captionof{table}{Algorithms used for comparison. All implementations from scikit-dimension library \citep{e23101368}.}
\label{tab:skdim}
\scalebox{0.88}{
\begin{tabular}{ | c || c | c |}
 \hline
 Name & Shortcut & citation \\ \hline\hline

CorrInt &COR &\cite{gassberger1983measuring} \\
MADA &MAD &\cite{farahmand2007manifold} \\
MLE &MLE &\cite{levina2004maximum} \\
lPCA &LPC &\cite{cangelosi2007component} \\
KNN &KNN &\cite{carter2009local} \\
DANCo &DAN &\cite{CERUTI20142569} \\
MiND\_ML &MIN &\cite{rozza2012novel} \\
ESS &ESS &\cite{johnsson2014low} \\
MOM &MOM &\cite{amsaleg2018extreme} \\
FisherS &FIS &\cite{albergante2019estimating} \\
TwoNN &TWO &\cite{facco2017estimating} \\
TLE &TLE &\cite{amsaleg2019intrinsic} \\

\hline\end{tabular}}
\end{table}

%% file: lid-ml-performance.tex
\subsection{LID and model performance}
\label{sec:performance}
In this section, we show, that LID estimates are connected with model behavior on some benchmark datasets for autoencoder and classification deep neural networks.
Our result suggests, that the connection between LID and model performance is significant, so LIDL estimates can potentially be used in problems like semi-supervised learning, active learning, uncertainty estimation, and curriculum learning.

\paragraph{Reconstruction error vs LID for autoencoders}
\begin{figure}[h]
    \centering
    \includegraphics[width=0.45\textwidth]{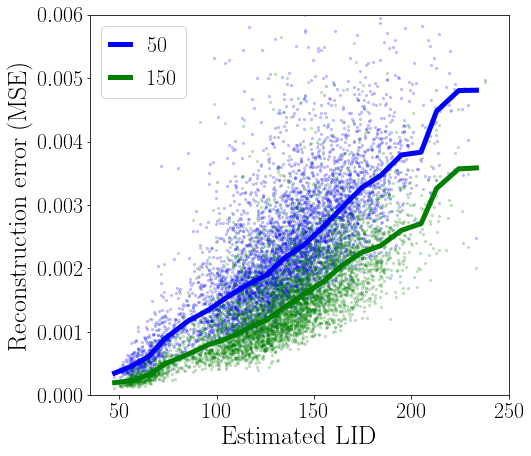}
    \caption{LIDL estimates and VAE MSE scatterplot for a sample from MNIST dataset. Lines are running medians for those point clouds. Values in legend are VAE latent space size.}
    \label{fig:lid-mse}
\end{figure}

Results from the previous section led us to next experiment, where we wanted to investigate if the estimate from LIDL is correlated with reconstruction error for the image in VAE \citep{kingma2014auto}. We trained VAE on MNIST with latent space sizes 50 and 150, and observed that there is a high correlation (Pearsons $R > 0.7$ in both cases) between MSE and LID. We plotted LIDL estimates against MSE for 5K images in Fig.~ \ref{fig:lid-mse}. We can see an almost linear relationship between those quantities.

\paragraph{LID and classification accuracy}

We observed that classifiers can achieve better accuracy on data points with lower LID estimates. We trained neural networks on a subset of MNIST (300 images) and FMNIST (50K images) datasets and noticed a negative correlation of LID and an accuracy on a test set. Results are presented in Fig~\ref{fig:lid-accuracy}. What is more, we observed similar behavior inside a majority of classes (9 classes for MNIST, and 8 for FMNIST). 
\begin{figure}[h]
    \begin{subfigure}{.5\textwidth}
      \centering
      \includegraphics[width=0.6\linewidth]{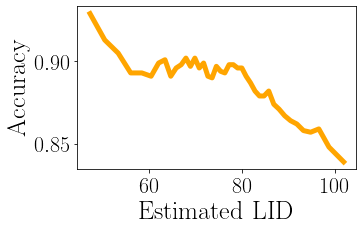}
    \end{subfigure}%
    \begin{subfigure}{.5
    \textwidth}
      \centering
      \includegraphics[width=0.6\linewidth]{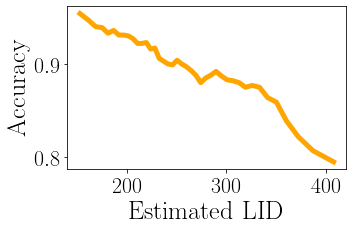}
    \end{subfigure}
    \caption{Average classifier accuracy per LID value on MNIST(left) and FMNIST(right) datasets. We can see a very strong negative correlation between those values.}
    \label{fig:lid-accuracy}
\end{figure}